\documentclass{article}

\usepackage[final]{nips_2018}

\usepackage[T1]{fontenc}
\usepackage[utf8]{inputenc}
\usepackage[french,english]{babel}
\usepackage[dvipsnames]{xcolor}
\usepackage{latexsym,amssymb,caption,listings}
\usepackage{natbib}
\setcitestyle{numbers}

\usepackage[colorlinks=true, linkcolor=red, citecolor=Aquamarine ]{hyperref}    
\usepackage{amsmath} 
\usepackage{enumitem} 
\usepackage{dsfont} 
\usepackage{graphicx} 
\usepackage{amsthm} 
\usepackage{subcaption}
\usepackage{amssymb} 
\usepackage{booktabs}       
\usepackage{amsfonts}       
\usepackage{nicefrac}       
\usepackage{microtype}      
\usepackage{amssymb}
\usepackage{etoolbox}
\usepackage{color}
\usepackage{wrapfig}
\usepackage{blindtext}
\usepackage[mathcal]{eucal}

\newcommand{\R}{\mathbb{R}}
\newcommand{\tr}{\mathrm{tr}}
\newcommand{\N}{\mathbb{N}}
\newcommand{\E}{\mathbb{E}}
\newcommand{\h}{\mathcal{H}}
\newcommand{\F}{\mathcal{F}}
\newcommand{\G}{\mathcal{G}}
\newcommand{\X}{\mathcal{X}}

\newcommand{\Z}{\mathbb{Z}}

\newcommand{\BIT}{\begin{itemize}}
\newcommand{\EIT}{\end{itemize}}
\newcommand{\BNUM}{\begin{enumerate}}
\newcommand{\ENUM}{\end{enumerate}}
\newcommand{\BA}{\begin{array}}
\newcommand{\EA}{\end{array}}

\newcommand{\mysec}[1]{Section~\ref{sec:#1}}
\newcommand{\eq}[1]{Eq.~(\ref{eq:#1})}

\renewcommand{\L}{\mathcal{L}}
\newcommand{\s}{S}
\newcommand{\C}{\Sigma}
\newcommand{\Sn}{\hat{\s}_n}
\newcommand{\Cn}{\hat{\Sigma}_n}
\newcommand{\yn}{\hat{y}}
\newcommand{\la}{\lambda}
\newcommand{\EE}{\mathcal{E}}
\newcommand{\Ltwo}{{L_2(d \rho_\X)}}
\newcommand{\rhox}{{\rho_\X}}
\renewcommand{\S}{{S}}
\newcommand{\frho}{{f_\rho}}
\newcommand{\expect}[1]{{\mathbb E}[#1]}
\newcommand{\Cnl}{\widehat{\C}_{n\lambda}}
\newcommand{\Cl}{\C_\la}
\newcommand{\Ll}{\L_\la}
\providecommand{\scal}[2]{\left\langle{#1},{#2}\right\rangle}

\newcommand{\eqals}[1]{\begin{align*}#1\end{align*}}
\newcommand{\eqal}[1]{\begin{align}#1\end{align}}

\newtheorem{defi}{Definition}
\newtheorem{thm}{Theorem}
\newcommand{\bt}{\begin{thm}}
\newcommand{\et}{\end{thm}}
\newtheorem{cor}{Corollary}
\newtheorem{prop}{Proposition}
\newtheorem{lemma}{Lemma}

\newtheorem{rmk}{Remark}
\newtheorem{example}{Example}

\makeatletter
\newtheorem{ass}{\bfseries (A \hspace*{.5mm} ) $\!\!\!\!\!\!\!\!\!\!\!\!\!\!\!$ }
\patchcmd{\endass}{\@endpefalse}{}{}{}
\newcommand{\bas}{\begin{ass}}
\newcommand{\eas}{\end{ass}}

\makeatother

\let\OLDthebibliography\thebibliography
\renewcommand\thebibliography[1]{
  \OLDthebibliography{#1}
  \setlength{\parskip}{2.4pt}
  \setlength{\itemsep}{0pt plus 0.3ex}
}

\title{Statistical Optimality of Stochastic Gradient Descent on Hard Learning Problems through Multiple Passes}

\author{
Loucas Pillaud-Vivien \\
INRIA - Ecole Normale Sup\'erieure \\
PSL Research University \\
\texttt{loucas.pillaud-vivien@inria.fr} \\
\And
Alessandro Rudi \\
INRIA - Ecole Normale Sup\'erieure \\
PSL Research University \\
\texttt{alessandro.rudi@inria.fr} \\
\And
Francis Bach \\
INRIA - Ecole Normale Sup\'erieure \\
PSL Research University \\
\texttt{francis.bach@inria.fr} \\
}

\newcounter{savenum}

\begin{document}

\maketitle

\begin{abstract}
We consider stochastic gradient descent (SGD) for least-squares regression with potentially several passes over the data. While several passes have been widely reported to perform practically better in terms of predictive performance on unseen data, the existing theoretical analysis of SGD suggests that a single pass is statistically optimal. While this is true for low-dimensional easy problems, we show that for hard problems, multiple passes lead to statistically optimal predictions while single pass does not; we also show that in these hard models, the optimal number of passes over the data increases with sample size. In order to define the notion of hardness and show that our predictive performances are optimal, we consider potentially infinite-dimensional models and notions typically associated to kernel methods, namely, the decay of eigenvalues of the covariance matrix of the features and the complexity of the optimal predictor as measured through the covariance matrix.
We illustrate our results on synthetic experiments with non-linear kernel methods and on a classical benchmark with a linear model.
\end{abstract}

\section{Introduction}

Stochastic gradient descent (SGD) and its multiple variants---averaged~\cite{polyak1992acceleration}, accelerated~\cite{lan2012optimal}, variance-reduced~\cite{roux2012stochastic,johnson2013accelerating,defazio2014saga}---are the workhorses of large-scale machine learning, because (a) these methods looks at only a few observations before updating the corresponding model, and (b) they are known in theory and in practice to generalize well to unseen data~\cite{bottou2018optimization}.

Beyond the choice of step-size (often referred to as the learning rate), the number of passes to make on the data remains an important practical and theoretical issue.  In the context of finite-dimensional models (least-squares regression or logistic regression), the theoretical answer has been known for many years: a single passes suffices for the optimal statistical performance~\cite{polyak1992acceleration,nemirovsky1983problem}. Worse, most of the theoretical work only apply to single pass algorithms, with some exceptions leading to analyses of multiple passes when the step-size is taken smaller than the best known setting~\cite{hardt2016train,Lin2017multiplepass}.

However, in practice, multiple passes are always performed as they empirically lead to better generalization (e.g., loss on unseen test data)~\cite{bottou2018optimization}.
But no analysis so far has been able to show that, given the appropriate step-size, multiple pass SGD was theoretically better than single pass SGD.

The main contribution of this paper is to show that for least-squares regression, while single pass averaged SGD is optimal for a certain class of ``easy'' problems, multiple passes are needed to reach optimal prediction performance on another class of ``hard'' problems.

In order to define and characterize these classes of problems, we need to use tools from infinite-dimensional models which are common in the analysis of kernel methods. De facto, our analysis will be done in infinite-dimensional feature spaces, and for finite-dimensional problems where the dimension far exceeds the number of samples, using these tools are the only way to obtain non-vacuous dimension-independent bounds. Thus, overall, our analysis applies both to finite-dimensional models with explicit features (parametric estimation), and to kernel methods (non-parametric estimation).
 
 The two important quantities  in the analysis are:
\BIT
\item[(a)] The decay of eigenvalues of the covariance matrix $\Sigma$ of the input features, so that the ordered eigenvalues $\lambda_m$ decay as $O(m^{-\alpha})$; the parameter $\alpha \geqslant 1$ characterizes the size of the feature space, $\alpha=1$ corresponding to the largest feature spaces and  $\alpha = +
\infty$ to finite-dimensional spaces. The decay will be measured through $\tr \Sigma^{1/\alpha} = \sum_m \lambda_m^{1/\alpha}$, which is small when the decay of eigenvalues is faster  than $O(m^{-\alpha})$.

 \item[(b)] The complexity of the optimal predictor $\theta_\ast$ as measured through the covariance matrix~$\Sigma$, that is with coefficients $\langle e_m, \theta_\ast\rangle$ in the eigenbasis $(e_m)_m$ of the covariance matrix that decay so that 
 $ \langle \theta_\ast, \Sigma^{1-2r} \theta_\ast\rangle$ is small. The parameter $r \geqslant 0$ characterizes the difficulty of the learning problem: $r = 1/2$ corresponds to characterizing the complexity of the predictor through the squared norm $\| \theta_\ast\|^2$, and thus $r$ close to zero corresponds to the hardest problems while $r$ larger, and in particular $r \geqslant 1/2$,    corresponds to simpler problems.
\EIT

 Dealing with non-parametric estimation provides a simple way to evaluate the optimality of learning procedures. Indeed, given problems with parameters $r$ and $\alpha$, the best prediction performance (averaged square loss on unseen data) is well known \cite{fischer2017sobolev} and decay as $O(n^{\frac{-2r\alpha}{2r\alpha +1}} )$, with $\alpha = +
\infty$ leading to the usual parametric rate $O(n^{-1})$. For \emph{easy problems}, that is  for which $r \geqslant \frac{\alpha-1}{2\alpha}$, then it is known that most iterative algorithms achieve this optimal rate of convergence (but with various running-time complexities), such as exact regularized risk minimization~\cite{caponnetto2007optimal}, gradient descent on the empirical risk~\cite{yao2007early}, or averaged stochastic gradient descent~\cite{dieuleveut2016nonparametric}.

  We show that for \emph{hard problems}, that is for which $r \leqslant \frac{\alpha-1}{2\alpha}$ (see Example \ref{ex:hard_problem} for a typical hard problem), then multiple passes are superior to a single pass. More precisely, under additional assumptions detailed in \mysec{ls} that will lead to a subset of the hard problems, with $\Theta(n^{(\alpha- 1 - 2 r\alpha)/( 1 + 2 r \alpha)})$ passes, we achieve the optimal statistical performance $O(n^{\frac{-2r\alpha}{2r\alpha +1}} )$, while for all other hard problems, a single pass only achieves $O(n^{-2r} )$.  This is illustrated in Figure~\ref{fig:optimality_zones}.

  We thus get a number of passes that grows with the number of observations $n$ and depends precisely on the quantities $r$ and $\alpha$. In synthetic experiments with kernel methods where $\alpha$ and $r$ are known, these scalings are precisely observed. In experiments on parametric models with large dimensions, we also exhibit an increasing number of required passes when the number of observations increases.

\begin{figure}[ht]
\footnotesize
\includegraphics[width=0.48\textwidth]{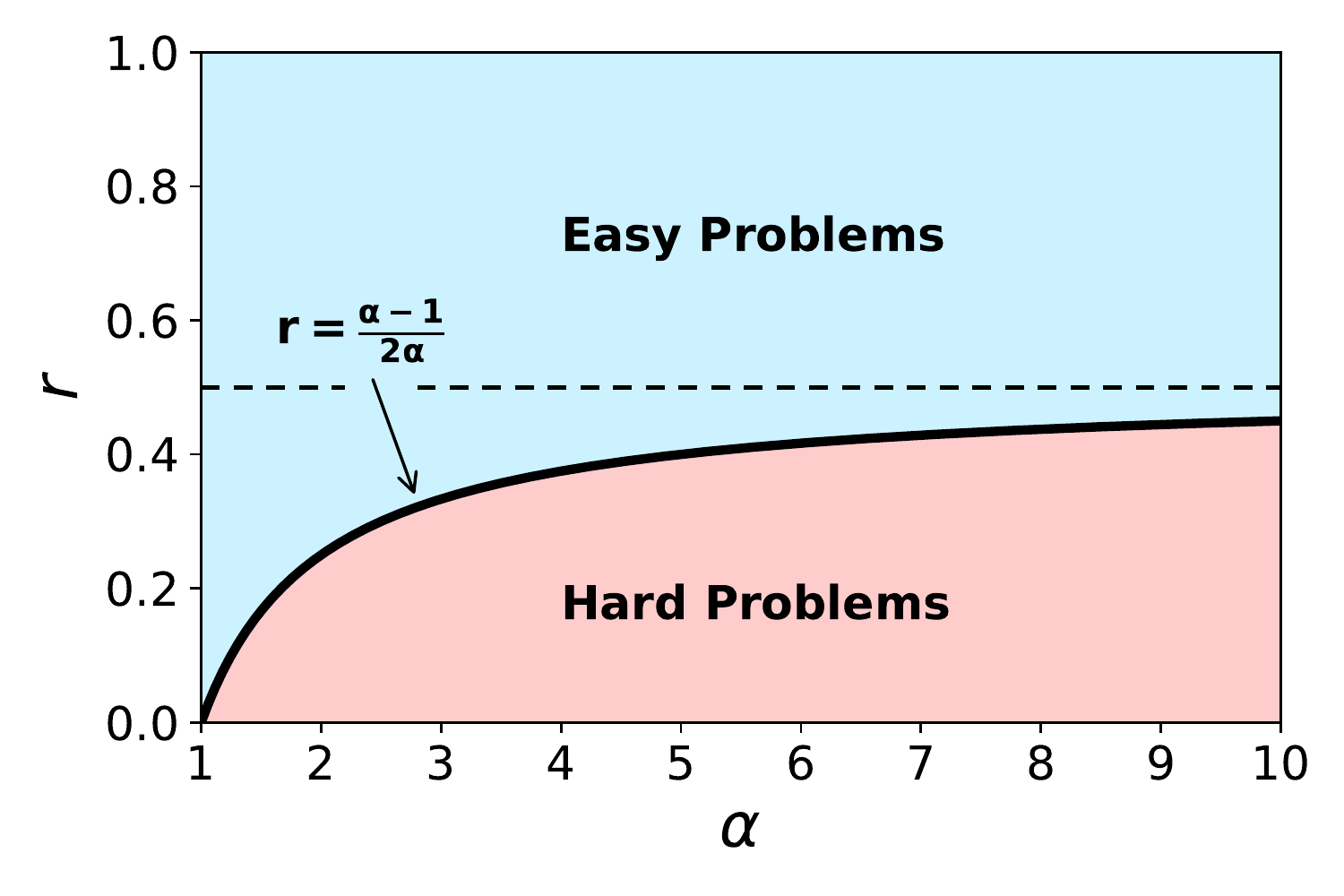}
\hspace{0.5cm}%
\includegraphics[width=0.48\textwidth]{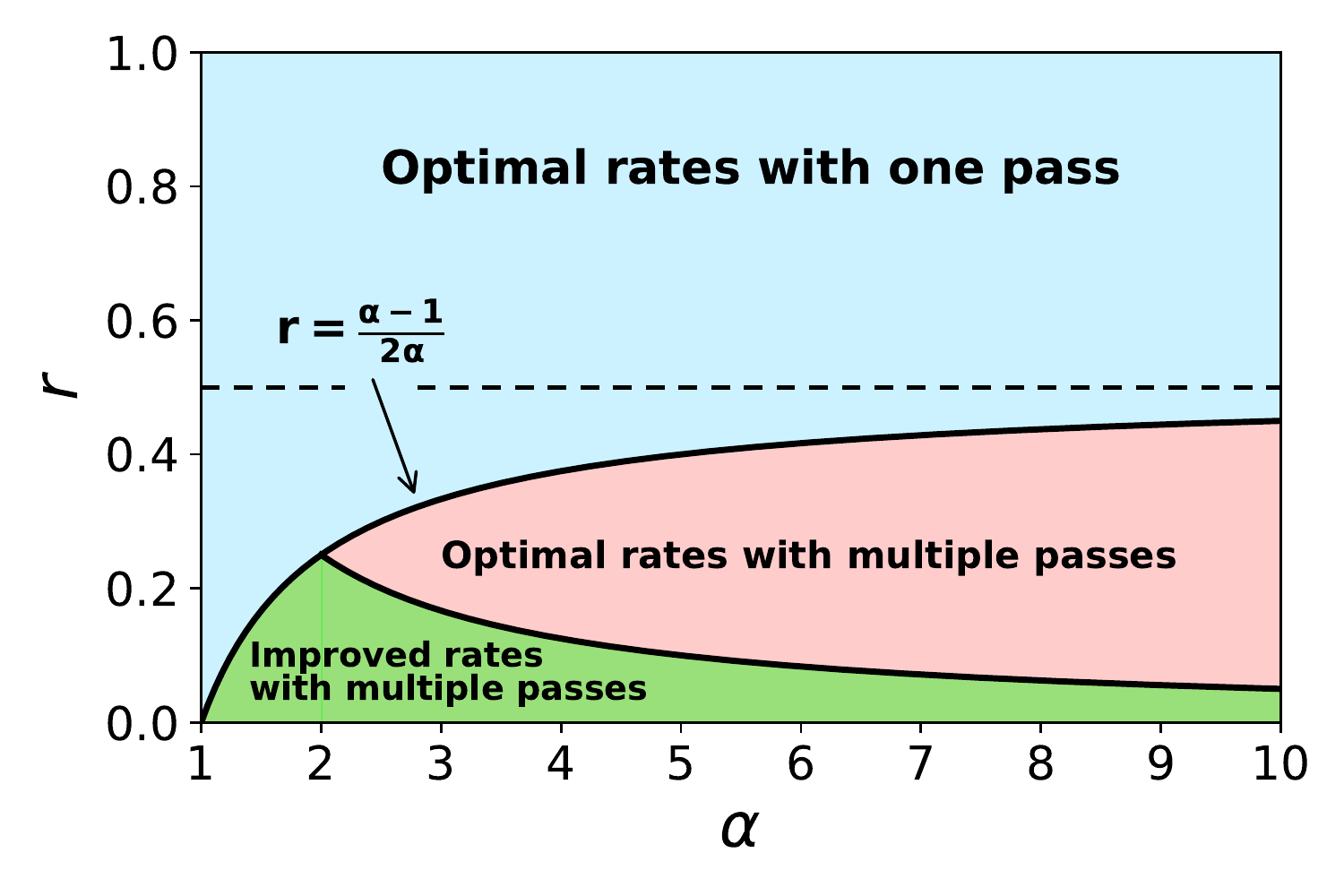} \\
\vspace{-0.5cm}
\caption{ (Left)  easy and hard problems in the $(\alpha,r)$-plane. (Right)  different regions for which multiple passes improved known previous bounds (green region) or reaches optimality (red region).}
\label{fig:optimality_zones}
\end{figure}

\section{Least-squares regression in finite dimension}
\label{sec:ls}

We consider a joint distribution $\rho$ on pairs of input/output $(x,y) \in \X \times \R$, where $\X$ is any input space, and we consider a feature map $\Phi$ from the input space $\X$ to a feature space $\mathcal{H}$, which we assume Euclidean in this section, so that all quantities are well-defined. In \mysec{kernels}, we will extend all the notions to Hilbert spaces.
 
\subsection{Main assumptions}\label{sect:assumptions}

We are considering predicting $y$ as a linear function  $f_\theta(x) = \langle \theta, \Phi(x) \rangle_\h$ of $\Phi(x)$, that is estimating $\theta \in \h$ such that $F(\theta) = \frac{1}{2} \E ( y - \langle \theta, \Phi(x) \rangle_\h)^2$ is as small as possible. 
Estimators will depend on $n$ observations, with standard sampling assumptions:

\begin{enumerate}[label={\bfseries(A\arabic*)}, ref=(A\arabic*),leftmargin=*]
\setcounter{enumi}{\value{savenum}}
\item \label{asm:iid}   \hspace*{.2cm} 
 \emph{The $n$ observations $(x_i,y_i) \in \X \times \R$, $i=1,\dots,n$, are independent and identically distributed from the distribution $\rho$.}
\setcounter{savenum}{\value{enumi}}
\end{enumerate}

Since $\h$ is finite-dimensional,  $F$ always has a (potentially non-unique) minimizer in $\h$ which we denote $\theta_\ast$.   We make the following standard boundedness assumptions:

\begin{enumerate}[label={\bfseries(A\arabic*)}, ref=(A\arabic*),leftmargin=*]
\setcounter{enumi}{\value{savenum}}
\item \label{asm:bounded} \hspace*{.2cm} \emph{$\| \Phi(x) \|\leqslant R$ almost surely, $| y -  \langle \theta_\ast, \Phi(x) \rangle_\h|$ is almost surely bounded by $\sigma$ and $|y|$ is almost surely bounded by $M$.}
\setcounter{savenum}{\value{enumi}}
\end{enumerate}

In order to obtain improved rates with multiple passes, and motivated by the equivalent previously used condition in reproducing kernel Hilbert spaces presented in \mysec{kernels}, we make the following extra assumption (we denote by $\Sigma = \E [ \Phi(x) \otimes_\h \Phi(x) ] $ the (non-centered) covariance matrix).

\begin{enumerate}[label={\bfseries(A\arabic*)}, ref=(A\arabic*),leftmargin=*]
\setcounter{enumi}{\value{savenum}}
\item \label{asm:hyp-space} \hspace*{.2cm} \emph{For $\mu \in [0,1]$, there exists $\kappa_\mu \geqslant 0$ such that, almost surely, $\Phi(x) \otimes_\h \Phi(x) \preccurlyeq_\h \kappa_\mu^2 R^{2\mu}  \Sigma^{1-\mu}$. Note that it can also be written as
 $\| \Sigma^{\mu/2 - 1/2} \Phi(x) \|_\h \leqslant \kappa_\mu R^\mu$.}
\setcounter{savenum}{\value{enumi}}
\end{enumerate}

Assumption \ref{asm:hyp-space} is always satisfied with any $\mu \in [0,1]$, and has particular values for $\mu=1$, with $\kappa_1 = 1$, and $\mu = 0$, where $\kappa_0$ has to be larger than the dimension of the space $\h$. 

We will also introduce a parameter $\alpha$ that characterizes the decay of eigenvalues of $\Sigma$ through the quantity $\tr \Sigma^{1/\alpha}$, as well as the difficulty of the learning problem through $ \| \Sigma^{1/2 - r} \theta_\ast \|_\h$, for $r \in [0,1]$. In the finite-dimensional case, these quantities can always be defined and most often finite, but may be very large compared to sample size. In the following assumptions the quantities are assumed to be finite and small compared to $n$.

\begin{enumerate}[label={\bfseries(A\arabic*)},ref=(A\arabic*),leftmargin=*]
\setcounter{enumi}{\value{savenum}}
\item \label{asm:capacity_condition} \hspace*{.2cm} \emph{There exists $\alpha > 1$ such that $\tr\ \Sigma ^{1/\alpha} < \infty$.}
\setcounter{savenum}{\value{enumi}}
\end{enumerate}

Assumption \ref{asm:capacity_condition} is often called the ``capacity condition''. First note that this assumption implies that the decreasing sequence of the eigenvalues of $\Sigma$, $(\lambda_m)_{m \geqslant 1}$, satisfies $\lambda_m = o \left( 1/m^{\alpha} \right)$. Note that
$\tr \Sigma^\mu\leqslant\kappa_\mu^2 R^{2\mu}  $ and thus often we have $\mu \geqslant 1/\alpha$, and in the most favorable cases in \mysec{kernels}, this bound will be achieved. We also assume:

\begin{enumerate}[label={\bfseries(A\arabic*)},ref=(A\arabic*),leftmargin=*]
\setcounter{enumi}{\value{savenum}}
\item \label{asm:source_condition} \hspace*{.2cm} \emph{There exists $r \geqslant 0$, such that    $ \| \Sigma^{1/2 - r} \theta_\ast \|_\h  < \infty $.}
\setcounter{savenum}{\value{enumi}}
\end{enumerate}

Assumption \ref{asm:source_condition} is often called the ``source condition''. Note also that for $r = 1/2$, this simply says that the optimal predictor has a small norm.

 In the subsequent sections, we essentially assume that $\alpha$, $\mu$ and $r$ are chosen (by the theoretical analysis, not by the algorithm) so that all quantities $R_\mu$, $\| \Sigma^{1/2 - r} \theta_\ast \|_\h$ and $\tr \Sigma^{1/\alpha}$ are finite and small. As recalled in the introduction, these parameters are often used in the non-parametric literature to quantify the hardness of the learning problem (Figure \ref{fig:optimality_zones}).

 We will use result with $O(\cdot)$ and $\Theta(\cdot)$ notations, which will all be independent of $n$ and $t$ (number of observations and number of iterations) but can depend on other finite constants. Explicit dependence on all parameters of the problem is given in proofs. More precisely, we will use the usual $O(\cdot)$ and $\Theta(\cdot)$ notations for sequences $b_{nt}$ and $a_{nt}$ that can depend on $n$ and $t$, as $a_{nt} = O(b_{nt})$ if and only if, there exists $M>0$ such that for all $n,t$, $a_{nt} \leqslant M b_{nt}$, and $a_{nt} = \Theta(b_{nt})$  if and only if, there exist $M,M' >0$ such that for all $n,t$, $M' b_{nt} \leqslant a_{nt} \leqslant M b_{nt}$.

\subsection{Related work}

Given our assumptions above, several algorithms have been developed for obtaining low values of the expected excess risk $\E \big[ F(\theta) \big] - 
F(\theta_\ast)$.

\textbf{Regularized empirical risk minimization.} Forming the empirical risk $\hat{F}(\theta)$, it minimizes $\hat{F}(\theta) + \lambda \| \theta\|_\h^2$, for appropriate values of $\lambda$. It is known that for easy problems where $r \geqslant \frac{\alpha-1}{2\alpha}$, it achieves the optimal rate of convergence
$O(n^{\frac{-2r\alpha}{2r\alpha +1}} )$~\cite{caponnetto2007optimal}. However, algorithmically, this requires to solve  a linear system of size $n$ times the dimension of $\h$. One could also use fast variance-reduced stochastic gradient algorithms such as SAG~\cite{roux2012stochastic}, SVRG~\cite{johnson2013accelerating} or SAGA~\cite{defazio2014saga}, with a complexity proportional to the dimension of $\h$ times $n + R^2 / \lambda$.

\textbf{Early-stopped gradient descent on the empirical risk.} Instead of solving the linear system directly, one can use gradient descent with early stopping~\cite{yao2007early,lin2018optimal}. Similarly to the regularized empirical risk minimization case, a rate of $O(n^{-\frac{2r\alpha}{2r\alpha + 1}})$ is achieved for the easy problems, where $r \geqslant \frac{\alpha-1}{2\alpha}$. Different iterative regularization techniques beyond batch gradient descent with early stopping have been considered, with computational complexities ranging from $O(n^{1+\frac{\alpha}{2r\alpha+1}})$ to $O(n^{1+\frac{\alpha}{4r\alpha+2}})$ times the dimension of $\h$ (or $n$ in the kernel case in \mysec{kernels}) for optimal predictions \cite{yao2007early,gerfo2008spectral,rosasco2015learning,blanchard2016convergence,lin2018optimal}.

\textbf{Stochastic gradient.} The usual stochastic gradient recursion is iterating from $i=1$ to $n$,
\[ \theta_i = \theta_{i-1} + \gamma \big( y_{i} - \langle \theta_{i-1}, \Phi(x_{i}) \rangle_\h \big) \Phi(x_{i}),
\]
with the averaged iterate $\bar{\theta}_n = \frac{1}{n} \sum_{i=1}^n \theta_i$. Starting from $\theta_0 = 0$,  \cite{bach2013non} shows that the expected excess performance $\E [ F(\bar{\theta}_n) ] - F(\theta_\ast)$ decomposes into a \emph{variance} term  that depends on the noise~$\sigma^2$ in the prediction problem, and a \emph{bias term}, that depends on the deviation $ \theta_\ast - \theta_0 = \theta_\ast$ between the initialization and the optimal predictor.  Their bound  is, up to universal constants,  $\frac{\sigma^2 {\rm dim}(\h)}{n} + \frac{ \| \theta_\ast \|_\h^2 }{\gamma n} $.

Further, \cite{dieuleveut2016nonparametric} considered the quantities $\alpha$ and $r$ above to get the bound, up to constant factors:
\[   \frac{\sigma^2 \tr \Sigma^{1/\alpha} ( \gamma n)^{1/\alpha} }{n} + \frac{ \| \Sigma^{1/2 -r } \theta_\ast \|^2 }{\gamma^{2r} n^{2r}}.
  \]
We recover the finite-dimensional bound for $\alpha = +\infty$ and $r=1/2$. The bounds above are valid for all $\alpha \geqslant 1$ and all $r \in [0,1]$, and the step-size $\gamma$ is such that $\gamma R^2 \leqslant 1/4$, and thus we see a natural trade-off appearing for the step-size $\gamma$, between bias and variance.

When $r \geqslant \frac{\alpha-1}{2\alpha}$, then the optimal step-size minimizing the bound above is $\gamma \propto n^{\frac{-2\alpha\min \{r,1\} - 1 +\alpha}{2\alpha\min \{r,1\} + 1}}$, and the obtained rate is optimal. Thus a single pass is optimal. However, when  $r \leqslant \frac{\alpha-1}{2\alpha}$, the best step-size does not depend on $n$, and one can only achieve $O(n^{-2r})$.

Finally, in the same multiple pass set-up as ours, \cite{Lin2017multiplepass} has shown that for easy problems where $r \geqslant \frac{\alpha-1}{2\alpha}$ (and single-pass averaged SGD is already optimal) that multiple-pass non-averaged SGD is becoming optimal after a correct number of passes (while single-pass is not). Our  proof principle of comparing to batch gradient is taken  from~\cite{Lin2017multiplepass}, but we apply it to harder problems where $r \leqslant \frac{\alpha-1}{2\alpha}$. Moreover we consider the multi-pass averaged-SGD algorithm, instead of non-averaged SGD, and take explicitly into account the effect of Assumption \ref{asm:hyp-space}.

\section{Averaged SGD with multiple passes}
\label{sec:ASGD}
We consider the following algorithm, which is stochastic gradient descent with sampling with replacement with multiple passes over the data (we experiment in Section \ref{app:experiments} of the Appendix with cycling over the data, with or without reshuffling between each pass).
\BIT
\item \textbf{Initialization}: $\theta_0 = \bar{\theta}_0 = 0$, $t$ = maximal number of iterations, $\gamma = 1/(4R^2) = $ step-size
\item \textbf{Iteration}: for $u=1$ to $t$, sample $i(u)$ uniformly from $\{1,\dots,n\}$ and make the step
\[
\theta_u = \theta_{u-1} + \gamma \big( y_{i(u)} - \langle \theta_{t-1}, \Phi(x_{i(u)}) \rangle_\h \big) \Phi(x_{i(u)}) \ \  \mbox{ and } 
\ \ \textstyle \bar{\theta}_u = ( 1 - \frac{1}{u}) \bar{\theta}_{u-1} + \frac{1}{u} \theta_u.
\]
\EIT
In this paper, following~\cite{bach2013non,dieuleveut2016nonparametric}, but as opposed to~\cite{dieuleveut2017harder}, we consider unregularized recursions. This removes a unnecessary regularization parameter (at the expense of harder proofs).

\subsection{Convergence rate and optimal number of passes}
Our main result is the following (see full proof in Appendix):
\begin{thm}
\label{thm:main_result}
Let $n \in \N^*$ and $t \geqslant n$, under Assumptions~\ref{asm:iid}, \ref{asm:bounded}, \ref{asm:hyp-space},~\ref{asm:capacity_condition},~\ref{asm:source_condition},~\ref{asm:reg-Pfrho}, with $\gamma = 1 / ( 4 R^2)$.
\begin{itemize}
\item For $\mu\alpha < 2r\alpha + 1 < \alpha $, if we take $t = \Theta(  n^{\alpha/\left(2r\alpha + 1 \right)})$, we obtain the following rate: 
\[ \E  F(\bar{\theta}_t) - F(\theta_\ast)  =  O ( n^{-2r\alpha/\left(2r\alpha + 1 \right)}) .\]
\item For $\mu\alpha \geqslant 2r\alpha + 1$, if we take $t = \Theta(  n^{1/\mu}~(\log  n)^{\frac{1}{\mu}})$, we obtain the following rate:
\[  \E  F(\bar{\theta}_t) - F(\theta_\ast)  \leqslant O (  n^{-2r/\mu} ). \]
\end{itemize}
\end{thm}

\paragraph{\bfseries Sketch of proof.} The main difficulty in extending proofs from the single pass case~\cite{bach2013non,dieuleveut2016nonparametric} is that as soon as an observation is processed twice, then statistical dependences are introduced and the proof does not go through. In a similar context, some authors have considered stability results~\cite{hardt2016train}, but the large step-sizes that we consider do not allow this technique. Rather, we follow \cite{rosasco2015learning,Lin2017multiplepass} and compare our multi-pass stochastic recursion $\theta_t$ to the batch gradient descent iterate $\eta_t$ defined as $ \eta_t  = \eta_{t-1} + \frac{\gamma}{n} \sum_{i=1}^n\big( y_{i} - \langle \eta_{t-1}, \Phi(x_{i}) \rangle_\h \big) \Phi(x_{i})$ with its averaged iterate $\bar{\eta}_t$. We thus need to study the predictive performance of $\bar{\eta}_t$ and the deviation $\bar{\theta}_t- \bar{\eta}_t$.
It turns out that, given the data, the deviation $\theta_t - \eta_t$ satisfies an SGD recursion (with the respect to the randomness of the sampling with replacement). For a more detailed summary of the proof technique see \mysec{proof}. 

The novelty compared to \cite{rosasco2015learning,Lin2017multiplepass} is (a) to use refined results on averaged SGD for least-squares, in particular convergence in various norms for the deviation $\bar{\theta}_t- \bar{\eta_t}$ (see \mysec{appsgd}), that can use our new Assumption \ref{asm:hyp-space}. Moreover, (b) we need to extend the convergence results for the batch gradient descent recursion from~\cite{lin2018optimal}, also to take into account the new assumption (see \mysec{ales}). These two results are interesting on their own.

\paragraph{\bfseries Improved rates with multiple passes.}
 We can draw the following conclusions:
\BIT
\item If $ 2\alpha r + 1 \geqslant \alpha $, that is, easy problems, it has been shown by \cite{dieuleveut2016nonparametric} that a single pass with a smaller step-size than the one we propose here is optimal, and our result does not apply.

\item If   $\mu\alpha < 2r\alpha + 1 < \alpha $, then our proposed   number of iterations is $t = \Theta(n^{\alpha / ( 2\alpha r + 1)})$, which is now greater than~$n$; the convergence rate is then $O(n^{\frac{-2r\alpha}{2r\alpha +1}} )$, and, as we will see in \mysec{lower}, the predictive performance is then optimal when $\mu \leqslant 2r$.

\item If  $\mu\alpha \geqslant 2r\alpha + 1$, then with a number of iterations is $t = \Theta(n^{1/\mu})$, which is greater than~$n$ (thus several passes), with a convergence rate equal to $O( n^{-2r / \mu})$, which improves upon the best known rates of $O(n^{-2r})$. As we will see in \mysec{lower}, this is  not optimal.
\EIT

Note that these rates are theoretically only  bounds on the optimal number of passes over the data, and one should be cautious when drawing conclusions; however our simulations on synthetic data, see Figure \ref{fig:t_versus_n} in Section \ref{sec:experiments}, confirm that our proposed scalings for the number of passes is observed in practice.

\section{Application to kernel methods}
\label{sec:kernels}

In the section above, we have assumed that $\h$ was finite-dimensional, so that the optimal predictor~$\theta_\ast \in \h$ was always defined. Note however, that our bounds that depends on $\alpha$, $r$ and $\mu$ are \emph{independent of the dimension}, and hence, intuitively, following~\cite{dieuleveut2017harder}, should apply immediately to infinite-dimensional spaces. 

We now first show in \mysec{hilbert} how this intuition can be formalized and how using kernel methods provides a particularly interesting example. Moreover, this interpretation allows to characterize the statistical optimality of our results in  \mysec{lower}.

\subsection{Extension to Hilbert spaces, kernel methods and non-parametric estimation}
\label{sec:hilbert}

Our main result in Theorem~\ref{thm:main_result} extends directly to the case where $\h$ is an infinite-dimensional Hilbert space. In particular, given a feature map $\Phi: \X \to \h$,  any vector $\theta \in \h$ is naturally associated to a function defined as $f_\theta(x) = \langle \theta, \Phi(x) \rangle_\h$. Algorithms can then be run with infinite-dimensional objects if the {\em kernel} $K(x',x) = \langle \Phi(x'), \Phi(x) \rangle_\h$ can be computed efficiently.  This identification of elements $\theta $ of $\h$ with functions $f_\theta$ endows the various quantities we have introduced in the previous sections, with natural interpretations in terms of functions. The stochastic gradient descent described in Section~\ref{sec:ASGD} adapts instantly to this new framework as the iterates $(\theta_u)_{u \leqslant t}$ are linear combinations of feature vectors $\Phi(x_i)$, $i=1,\dots,n$, and the algorithms can classically be ``kernelized''~\cite{Ying2008,dieuleveut2016nonparametric}, with an overall running time complexity of $O(nt)$.

First note that Assumption~\ref{asm:hyp-space} is equivalent
to, for all $x \in \X$ and $\theta \in \h$, $|f_\theta(x)|^2 \leqslant \kappa_\mu^2 R^{2\mu} \langle f_\theta, \Sigma^{1-\mu} f_\theta \rangle_\h$, that is, 
$\| g\|^2_{L_\infty} \leqslant \kappa_\mu^2 R^{2\mu} \| \Sigma^{1/2 - \mu/2 }g \|_\h^2$ for any $g \in \h$ and also implies\footnote{Indeed, for any $g \in \h$, $\| \Sigma^{1/2 - \mu/2 }g \|_\h = \| \Sigma^{- \mu/2 }g \|_{L_2} \leqslant \| \Sigma^{- 1/2 }g \|^{\mu}_{L_2} \| g \|^{1-\mu}_{L_2} = \| g \|^{\mu}_\h \| g \|^{1-\mu}_{L_2}$, where we used that for any $g \in \h$, any bounded operator $A$, $s \in [0,1]$: $\|A^s g\|_{L_2} \leqslant \|A g\|_{L_2}^s \|g\|_{L_2}^{1-s}$ (see \cite{rudi2017generalization}).}
$\| g\|_{L_\infty} \leqslant  {\kappa_\mu} R^{\mu}  \| g \|_\h^{\mu} \| g\|_{L_2}^{1-\mu}$, which are common assumptions in the context of kernel methods~\cite{steinwart2009optimal}, essentially controlling in a more refined way the regularity of the whole space of functions associated to $\h$, with respect to the $L^\infty$-norm, compared to the too crude inequality $\|g\|_{L^\infty} = \sup_x |\scal{\Phi(x)}{g}_\h| \leqslant \sup_x \|\Phi(x)\|_\h \|g\|_\h \leqslant R\|g\|_\h$.

The natural relation with functions allows to analyze effects that are crucial in the context of learning, but difficult to grasp in the finite-dimensional setting. Consider the following prototypical example of a hard learning problem,

\begin{example}[Prototypical hard problem on simple Sobolev space]
\label{ex:hard_problem}
Let $\X = [0,1]$, with $x$ sampled uniformly on $X$ and
\[
y = \textrm{sign}(x-1/2) + \epsilon, \quad \Phi(x) = \{ |k|^{-1} e^{2 i k \pi x} \}_{ k \in \mathbb{Z}^\ast }.
\]
\end{example}

This corresponds to the kernel $K(x,y) = \sum_{k \in \mathbb{Z}^\ast}  |k|^{-2} e^{2 i k \pi (x-y)}$, which is well defined (and lead to the simplest Sobolev space). Note that for any $\theta \in \h$, which is here identified as the space of square-summable sequences $\ell^2(\mathbb{Z})$, we have $f_\theta(x) = \langle \theta, \Phi(x) \rangle_{\ell^2(\mathbb{Z})}  = \sum_{k \in \mathbb{Z}^\ast} \frac{\theta_k}{|k|} e^{2 i k \pi x}  $. This means that for any estimator $\hat{\theta}$ given by the algorithm, $f_{\hat{\theta}}$ is at least once continuously differentiable, while the target function $\textrm{sign}(\cdot-1/2)$ is not even continuous. Hence, we are in a situation where~$\theta_*$, the minimizer of the excess risk, does not belong to $\h$. Indeed let represent $\textrm{sign}(\cdot-1/2)$ in $\h$, for almost all $x \in [0,1]$, by its Fourier series
$\textrm{sign}(x-1/2) = \sum_{k \in \mathbb{Z}_\ast} \alpha_k e^{2 i k \pi x}$, with 
$|\alpha_k| \sim 1/ k$, an informal reasoning would lead to $(\theta_\ast)_k = \alpha_k |k| \sim 1$, which is not square-summable and thus $\theta_* \notin \h$. For more details, see~\cite{Adams1975Sobolev,Wahba1990spline}.

This setting generalizes important properties that are valid for Sobolev spaces, as shown in the following example, where $\alpha, r, \mu$ are characterized in terms of the smoothness of the functions in $\h$, the smoothness of $f^*$ and the dimensionality of the input space $\X$.

\begin{example}[Sobolev Spaces \cite{wendland2004scattered,steinwart2009optimal,bach2017equivalence, fischer2017sobolev}]
	Let $\X \subseteq \R^d$, $d \in \N$, with $\rhox$ supported on $\X$, absolutely continous with the uniform distribution and such that $\rhox(x) \geqslant a >0$ almost everywhere, for a given $a$.  Assume that $f^*(x) = \expect{y|x}$ is $s$-times differentiable, with $s > 0$. Choose a kernel, inducing Sobolev spaces of smoothness $m$ with $m >d/2$, as the Mat\'ern kernel
	\begin{align*} 
	K(x',x) = \|x'-x\|^{m-d/2} {\mathcal {K}}_{d/2-m}(\|x'-x\|),
	\end{align*}
	where ${\mathcal {K}}_{d/2-m}$ is the modified Bessel function of the second kind.
	Then the assumptions are satisfied for any $\epsilon > 0$, with $
	\alpha = \frac{2m}{d}, ~~~ \mu = \frac{d}{2m} + \epsilon, ~~~r = \frac{s}{2m}.$
\end{example}
In the following subsection we compare the rates obtained in Thm.~\ref{thm:main_result}, with known lower bounds under the same assumptions.

\subsection{Minimax lower bounds}
\label{sec:lower}
In this section we recall known lower bounds on the rates for classes of learning problems satisfying the conditions in Sect.~\ref{sect:assumptions}. Interestingly, the comparison below shows that our results in Theorem~\ref{thm:main_result} are optimal in the setting $2r \geqslant \mu$. While the optimality of SGD was known for the regime $\{2r\alpha + 1 \geqslant \alpha~ \cap ~2r \geqslant \mu\}$, here we extend the optimality to the new regime $\alpha \geqslant 2r\alpha + 1 \geqslant \mu \alpha$, covering essentially all the region $2r \geqslant \mu$, as it is possible to observe in Figure~\ref{fig:optimality_zones}, where for clarity we plotted the best possible value for $\mu$ that is $\mu = 1/\alpha$  \cite{fischer2017sobolev} (which is true for Sobolev spaces).

When $r \in (0,1]$ is fixed, but there are no assumptions on $\alpha$ or $\mu$, then the optimal minimax rate of convergence
is $O(n^{-2r / ( 2r + 1)})$, attained by regularized empirical risk minimization~\cite{caponnetto2007optimal} and other spectral filters on the empirical covariance operator~\cite{blanchard2017optimal}.

When $r \in (0,1]$ and $\alpha \geqslant 1$ are fixed (but there are no constraints on $\mu$), the optimal minimax rate of convergence $O(n^{\frac{-2r\alpha}{2r\alpha +1}} )$ is attained when $r \geqslant \frac{\alpha-1}{2\alpha}$, with empirical risk minimization~\cite{lin2018optimal} or stochastic gradient descent~\cite{dieuleveut2016nonparametric}.

When $r \geqslant \frac{\alpha-1}{2\alpha}$, the rate of convergence $O(n^{\frac{-2r\alpha}{2r\alpha +1}} )$  is known to be a lower bound on the optimal minimax rate, but the best upper-bound so far is $O(n^{-2r})$ and is achieved by empirical risk minimization~\cite{lin2018optimal} or stochastic gradient descent~\cite{dieuleveut2016nonparametric}, and the optimal rate is not known.

When $r \in (0,1]$,  $\alpha \geqslant 1$ and $\mu \in [1/\alpha, 1]$ are fixed, then 
 the rate of convergence $O(n^{\frac{-\max\{\mu, 2r\} \alpha}{2\max\{\mu, 2r\}\alpha +1}} )$  is known to be a lower bound on the optimal minimax rate~\cite{fischer2017sobolev}. This is attained by regularized empirical risk minimization when $2r \geqslant \mu$~\cite{fischer2017sobolev}, and \emph{now by SGD with multiple passes}, and it is thus the optimal rate in this situation. When $2r < \mu$, the only known upper bound is $O(n^{- 2 \alpha r / ( \mu \alpha + 1) })$, and the optimal rate is not known.

 \section{Experiments}
\label{sec:experiments}
In our experiments, the main goal is to show that with more that one pass over the data, we can improve the accuracy of SGD when the problem is hard. We also want to highlight our dependence of the optimal number of passes (that is $t/n$) with respect to the number of observations $n$. 

\paragraph{\bfseries Synthetic experiments.} Our main experiments are performed on artificial data following the setting in \cite{rudi2017generalization}. For this purpose, we take kernels $K$ corresponding to splines of order $q$ (see \cite{Wahba1990spline}) that fulfill Assumptions \ref{asm:iid} \ref{asm:bounded} \ref{asm:hyp-space} \ref{asm:capacity_condition} \ref{asm:source_condition} \ref{asm:reg-Pfrho}. Indeed, let us consider the following function
\[   
\Lambda_q(x,z) =\sum_{k \in \Z}\frac{e^{2 i \pi k (x-z)}}{\left|k\right|^q},
\]
defined almost everywhere on $[0,1]$, with $q \in \R$, and for which we have the interesting relationship: $\langle \Lambda_q(x,\cdot), \Lambda_{q'}(z,\cdot) \rangle_{\Ltwo} = \Lambda_{q + q'}(x,z)$ for any $q,q' \in \R$. Our setting is the following:
\begin{itemize}
\item {\bfseries Input distribution:} $\X = [0,1]$ and $\rho_\X$ is the uniform distribution.
\item {\bfseries Kernel:} $\forall (x,z) \in [0,1], \ K(x,z) = \Lambda_{\alpha} (x,z)$.
\item {\bfseries Target function:} $\forall x \in [0,1], \ \theta_\ast = \Lambda_{r\alpha + \frac{1}{2}} (x,0)$.
\item {\bfseries Output distribution :} $\rho(y|x)$ is a Gaussian with variance $\sigma^2$ and mean $\theta_\ast$.
\end{itemize}
For this setting we can show that the learning problem satisfies Assumptions \ref{asm:iid} \ref{asm:bounded} \ref{asm:hyp-space} \ref{asm:capacity_condition} \ref{asm:source_condition} \ref{asm:reg-Pfrho} with $r,\ \alpha,\ \textrm{and} \mu = 1/\alpha $.
We take different values of these parameters to encounter all the different regimes of the problems shown in Figure \ref{fig:optimality_zones}. 

For each $n$ from $100$ to $1000$, we found the optimal number of steps $t_*(n)$ that minimizes the test error $F(\bar{\theta}_t) - F(\theta_\ast)$. Note that because of overfitting the test error increases for $t > t_*(n)$. In Figure \ref{fig:t_versus_n}, we show $t_*(n)$ with respect to $n$ in $\log$ scale. As expected, for the easy problems (where $r \geqslant \frac{\alpha-1}{2\alpha}$, see top left and right plots), the slope of the plot is $1$ as one pass over the data is enough: $t_*(n) = \Theta(n)$.  But we see that for hard problems (where $r \leqslant \frac{\alpha-1}{2\alpha}$, see bottom left and right plots), we need more than one pass to achieve optimality as the optimal number of iterations is very close to $ t_*(n) = \Theta\left(n^{\frac{\alpha}{2r\alpha + 1}}\right) $. That matches the theoretical predictions of Theorem \ref{thm:main_result}. We also notice in the plots that, the bigger $\frac{\alpha}{2r\alpha + 1}$ the harder the problem is and the bigger the number of epochs we have to take. Note, that to reduce the noise on the estimation of $t_*(n)$, plots show an average over 100 replications. 

To conclude, the experiments presented in the section correspond exactly to the theoretical setting of the article (sampling with replacement), however we present in Figures \ref{fig:t_versus_n_cycling} and \ref{fig:t_versus_n_without_replacement} of Section \ref{app:experiments} of the Appendix results on the same datasets for two different ways of sampling the data:
(a)\emph{without replacement}: for which we select randomly the data points but never use twice the same point in one epoch, (b) \emph{cycles}: for which we pick successively the data points in the same order. The obtained scalings relating number of iterations or passes to number of observations are the same.
\begin{figure}[ht]
\footnotesize
\includegraphics[width=0.48\textwidth]{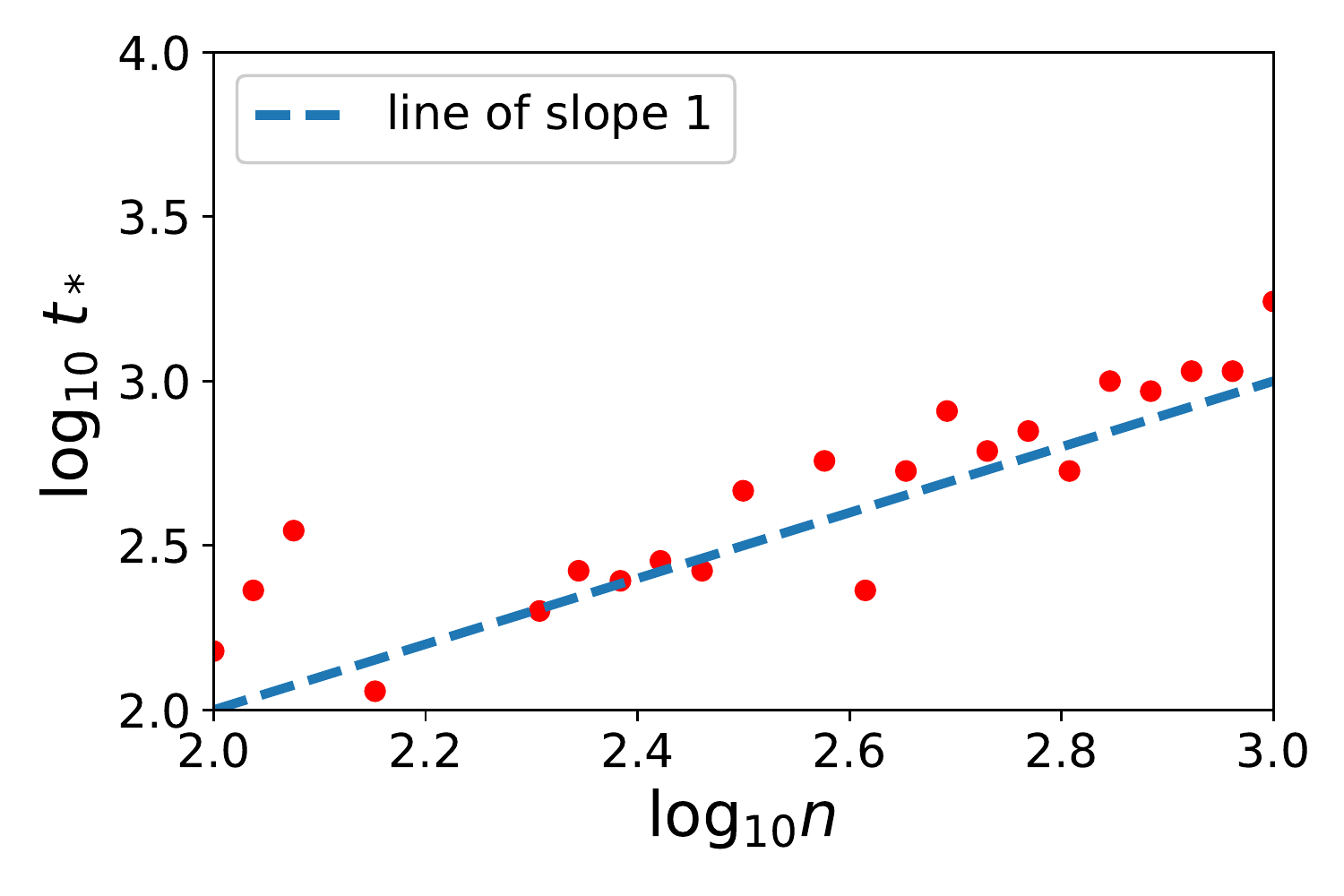}
\hspace{0.5cm}%
\includegraphics[width=0.48\textwidth]{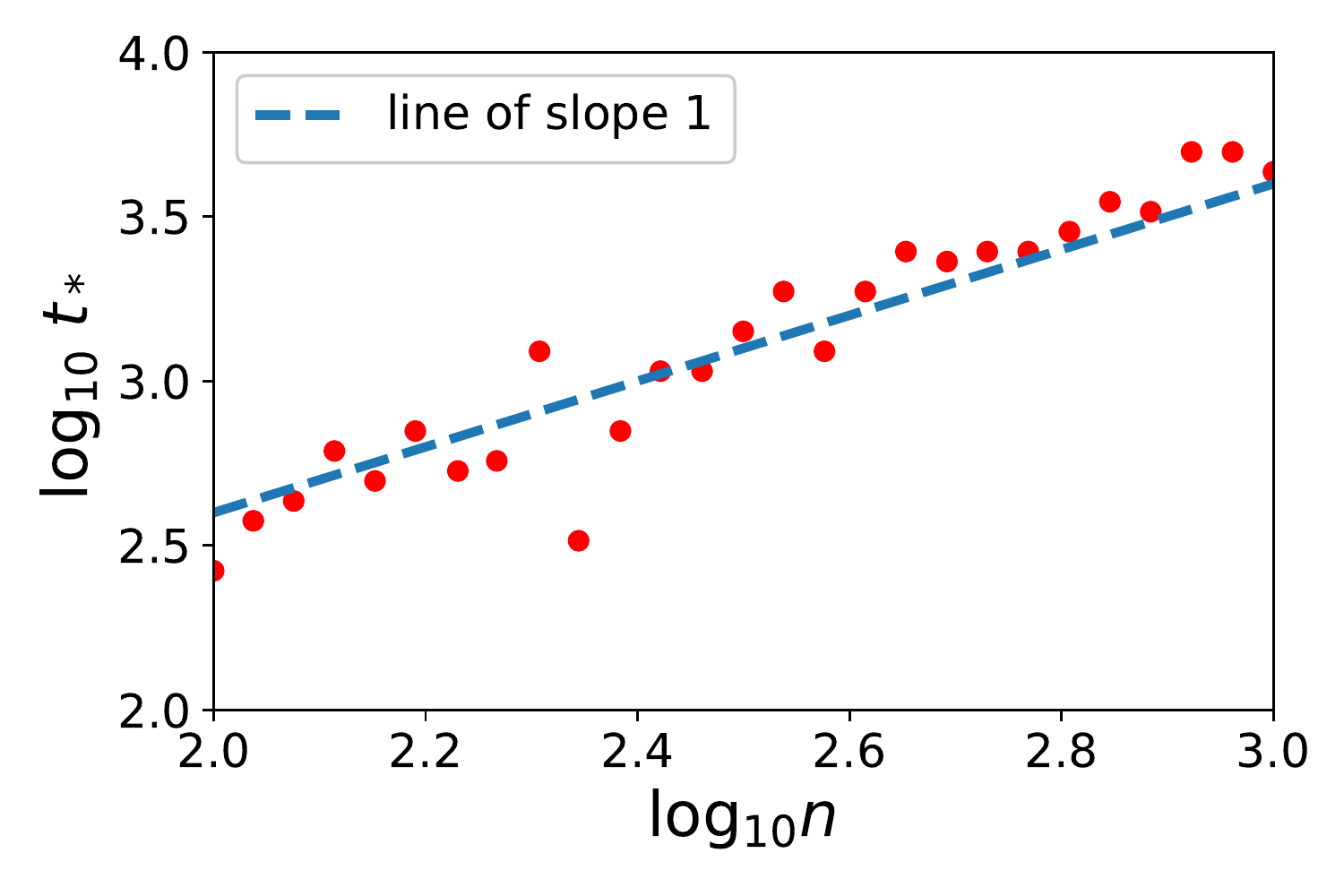} \\
\includegraphics[width=0.48\textwidth]{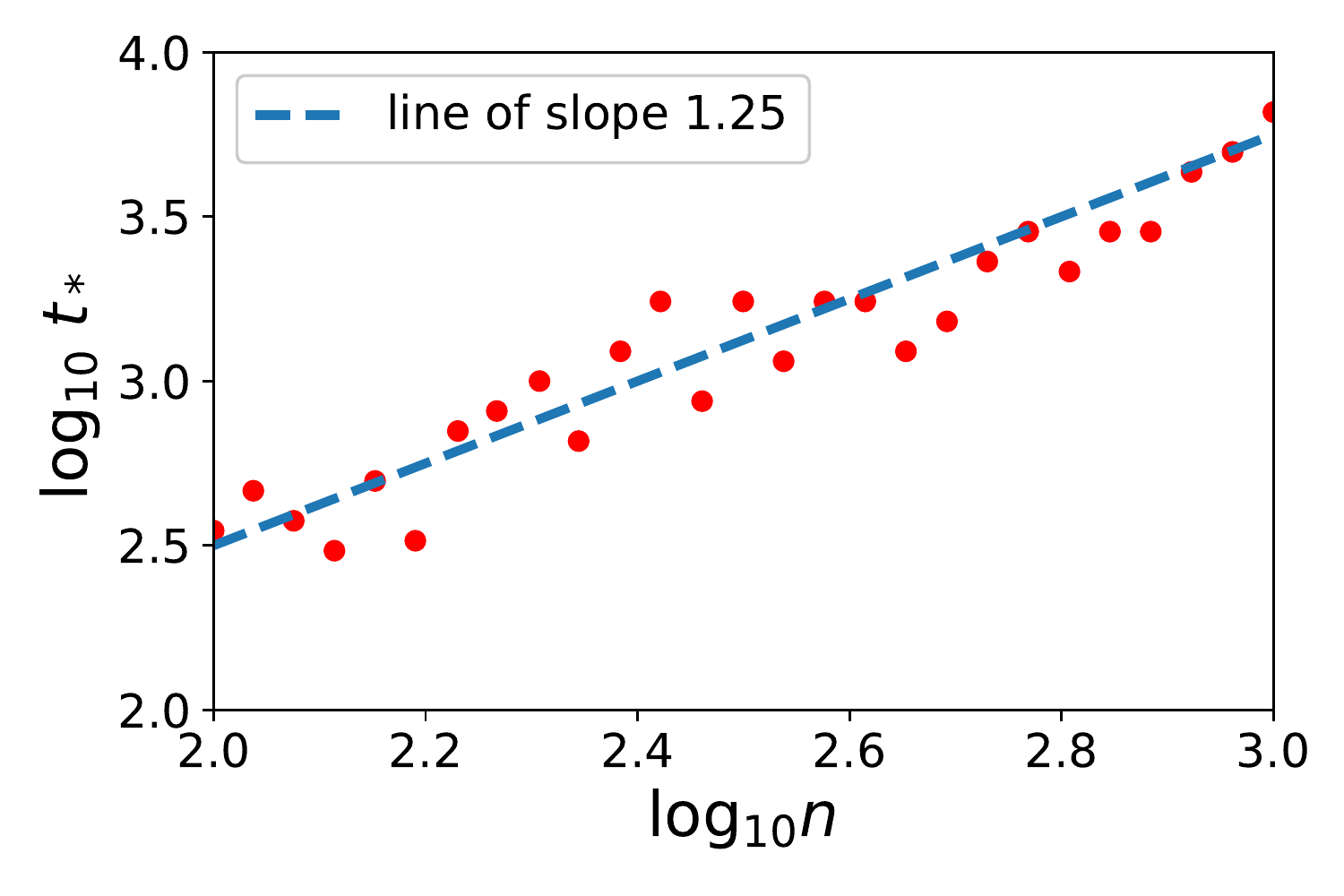}
\hspace{0.5cm}%
\includegraphics[width=0.48\textwidth]{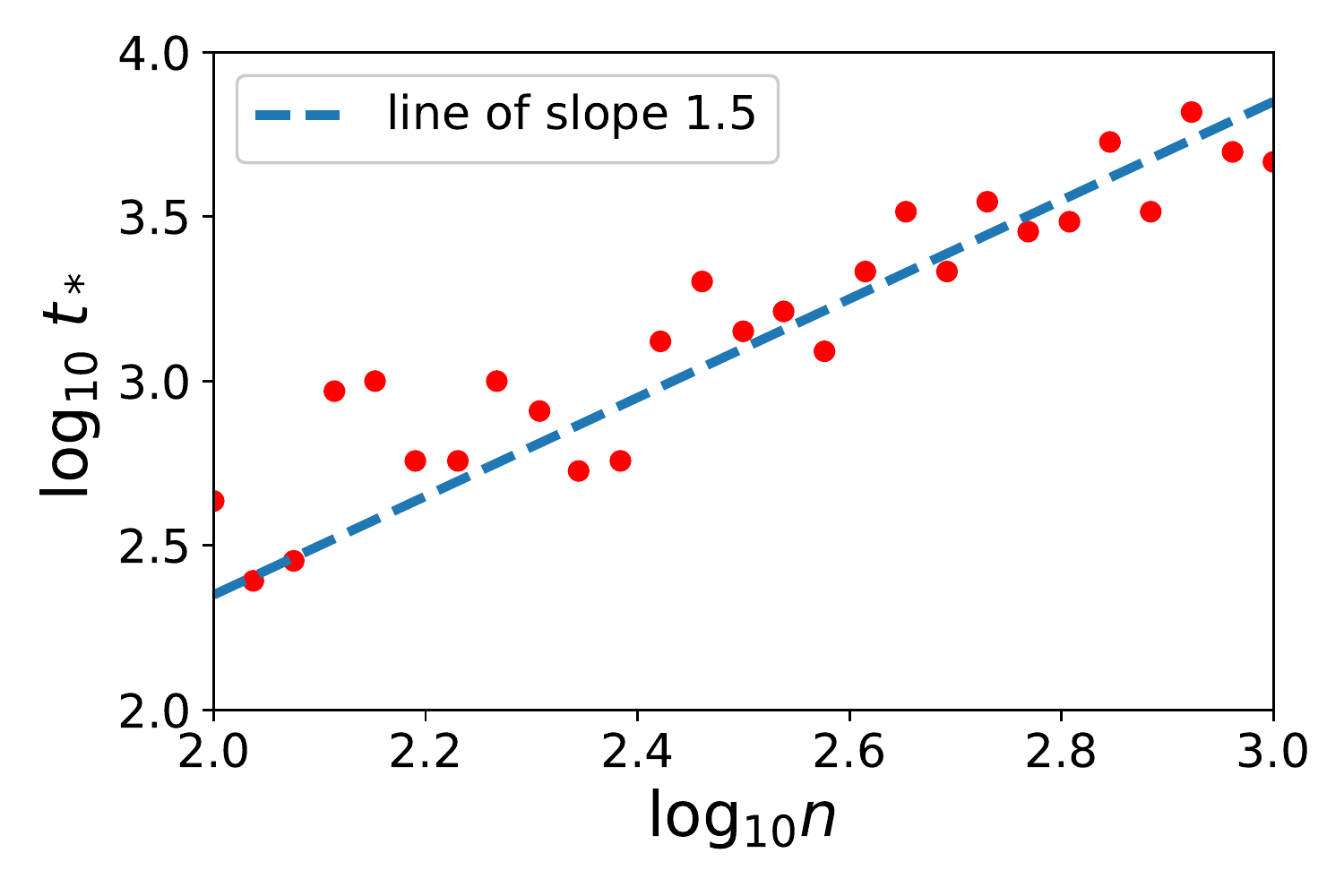}
\vspace{-0.6cm}
\caption{ \small The four plots represent each a different configuration on the $(\alpha , r)$ plan represented in Figure \ref{fig:optimality_zones}, for $r = 1/(2\alpha) $. {\bfseries Top left} ($\alpha = 1.5$) and {\bfseries right} ($\alpha = 2$) are two easy problems (Top right is the limiting case where $r = \frac{\alpha-1}{2\alpha}$) for which one pass over the data is optimal. {\bfseries Bottom left} ($\alpha = 2.5$) and {\bfseries right} ($\alpha = 3$) are two hard problems for which an increasing number of passes is required. The blue dotted line are the slopes predicted by the theoretical result in Theorem \ref{thm:main_result}. }
\label{fig:t_versus_n}
\end{figure}

\paragraph{\bfseries Linear model.} To illustrate our result with some real data, we show how the optimal number of passes over the data increases with the number of samples. In Figure \ref{fig:MNIST}, we simply performed linear least-squares regression on the MNIST dataset and plotted the optimal number of passes over the data that leads to the smallest error on the test set. Evaluating $\alpha$ and $r$ from Assumptions \ref{asm:capacity_condition} and \ref{asm:source_condition}, we found $\alpha = 1.7$ and $r = 0.18$. As $r = 0.18 \leqslant \frac{\alpha - 1}{2\alpha} \sim 0.2$, Theorem \ref{thm:main_result} indicates that this corresponds to a situation where only one pass on the data is not enough, confirming the behavior of Figure \ref{fig:MNIST}. This suggests that learning MNIST with linear regression is a \textit{hard problem}.

\begin{figure}[ht]
\center
\includegraphics[width=0.5\textwidth]{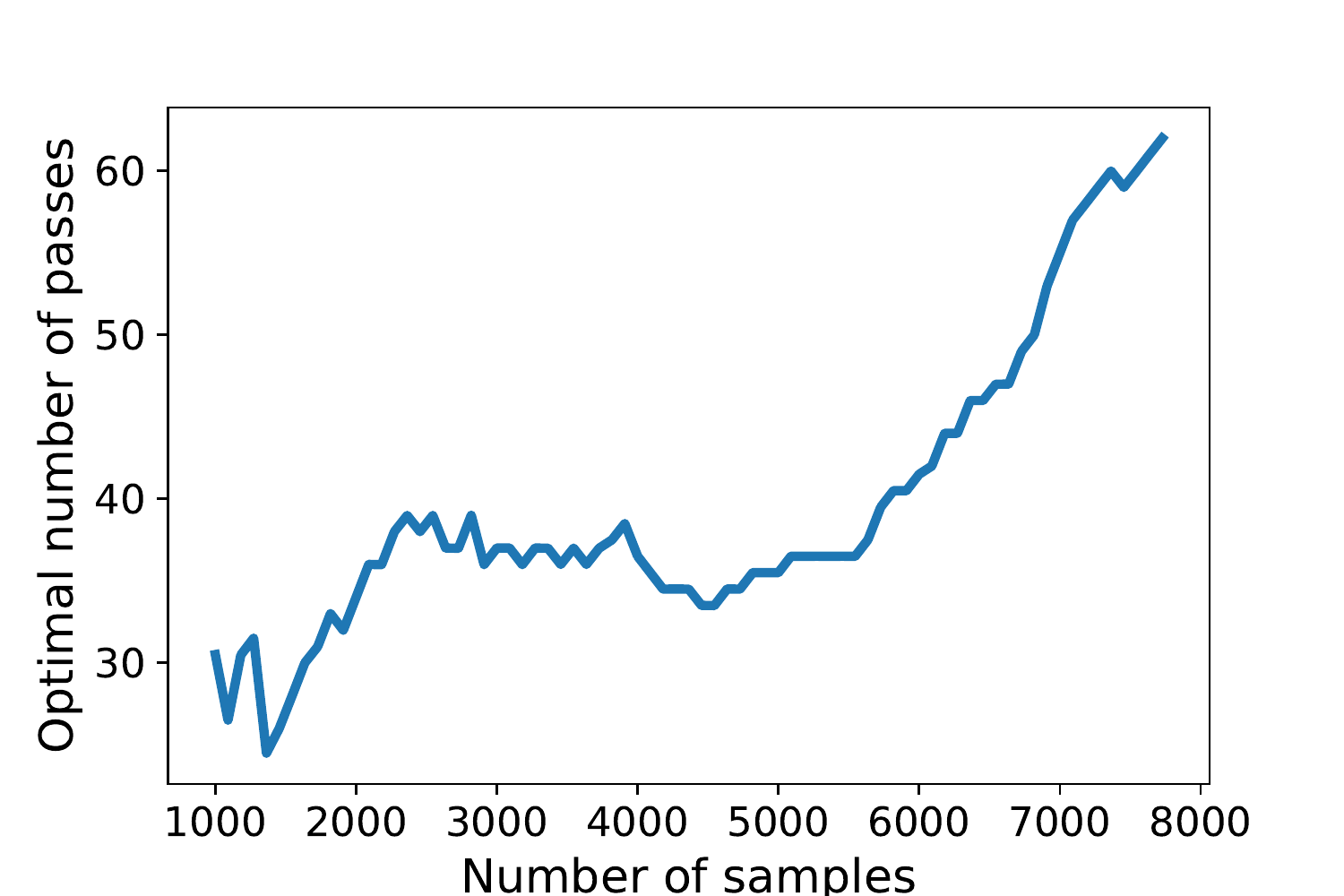}
\caption{ \small For the MNIST data set, we show the optimal number of passes over the data with respect to the number of samples in the case of the linear regression.}
\label{fig:MNIST}
\end{figure}

\section{Conclusion}

In this paper, we have shown that for least-squares regression, in hard problems where single-pass SGD is not statistically optimal ($r <  \frac{\alpha-1}{2\alpha}$), then multiple passes lead to statistical optimality with a number of passes that somewhat surprisingly needs to  grow with sample size, with a convergence rate which is superior to previous analyses of stochastic gradient. Using a non-parametric estimation, we show that under certain conditions ($2r \geqslant \mu$), we attain statistical optimality.

Our work could be extended in several ways: (a) our experiments suggest that cycling over the data and cycling with random reshuffling perform similarly to sampling with replacement, it would be interesting to combine our theoretical analysis with work aiming at analyzing other sampling schemes~\cite{NIPS2016_6245,gurbuzbalaban2015random}.  (b) Mini-batches could be also considered with a potentially interesting effects compared to the streaming setting. Also, 
(c) our analysis focuses on least-squares regression, an extension to all smooth loss  functions would widen its applicability.
Moreover, (d) providing optimal efficient algorithms for the situation $2r < \mu$ is a clear open problem (for which the optimal rate is not known, even for non-efficient algorithms).   Additionally, (e) in the context of classification, we could combine our analysis with~\cite{loucas} to study the potential discrepancies between training and testing losses and errors when considering high-dimensional models~\cite{zhang2016understanding}. More generally, (f) we could explore the effect of our analysis for methods based on the least squares estimator in the context of structured prediction \cite{ciliberto2016consistent,osokin2017structured,ciliberto2018localized} and (non-linear) multitask learning \cite{ciliberto2017consistent}. Finally, (g) to reduce the computational complexity of the algorithm, while retaining the (optimal) statistical guarantees, we could combine multi-pass stochastic gradient descent, with approximation techniques like {\em random features} \cite{rahimi2008random}, extending the analysis of \cite{carratino2018learning} to the more general setting considered in this paper.

\subsection*{Acknowledgements}
We acknowledge support from the European Research Council (grant SEQUOIA 724063). We also thank Raphaël Berthier and Yann Labbé for their enlightening advices on this project.

\bibliographystyle{unsrt}
\bibliography{multipass_sgd}

\clearpage

\appendix

{\bfseries \huge Appendix}

\vspace{0.5cm}

The appendix in constructed as follows:

\BIT
\item
We first present in Section \ref{sec:appsgd} a new result for stochastic gradient recursions which generalizes the work of ~\cite{bach2013non} and~\cite{dieuleveut2016nonparametric} to more general norms. This result could be used in other contexts.

\item The proof technique for Theorem~\ref{thm:main_result} is presented in \mysec{proof}.

\item In \mysec{SGD-BGD} we give a proof of the various lemmas needed in the first part of the proof of Theorem~\ref{thm:main_result} (deviation between SGD and batch gradient descent).

\item In \mysec{ales} we provide new results for the analysis of batch gradient descent, which are adapted to our new \ref{asm:hyp-space}, and instrumental in proving Theorem~\ref{thm:main_result} in \mysec{proof}.

\item Finally, in Section \ref{app:experiments} we present experiments for different sampling techniques.

\EIT

\section{A general result for the SGD variance term}
\label{sec:appsgd}
Independently of the problem studied in this paper, we consider i.i.d.~observations $(z_t,\xi_t) \in \h \times \h$ a Hilbert space, and the recursion started from $\mu_0=0$.
\begin{align}
\label{eq:sgdmu}\mu_t = \left(I - \gamma z_t \otimes z_t \right)\mu_{t-1} + \gamma  \xi_t
\end{align}
 (this will applied with $z_t = \Phi(x_{i(t)})$). This corresponds to the variance term of SGD.  We denote by $\bar{\mu}_t$ the averaged iterate
$\bar{\mu}_t = \frac{1}{t} \sum_{i=1}^t \mu_i$.

The goal of the proposition below is to provide a bound on $\E \left[  \left\| H^{u/2} \bar{\mu}_t  \right\|^2  \right]$ for $u \in [0, \frac{1}{\alpha}+1]$, where $H = \E \left[z_t \otimes z_t\right]$ is such that $\tr H^{1/\alpha} $ is finite. Existing results only cover the case $u = 1$.
\begin{prop}[A general result for the SGD variance term]
\label{prop:generalSGD}
Let us consider the  recursion in \eq{sgdmu} started at $\mu_0 = 0$. Denote $\E \left[z_t \otimes z_t\right] = H$, assume that  $\tr H^{1/\alpha} $ is finite, $\E\left[ \xi_t \right] = 0$, $\E\left[ (z_t \otimes z_t)^2 \right] \preccurlyeq R^2 H$, $\E\left[ \xi_t \otimes \xi_t\right] \preccurlyeq \sigma^2 H$ and $\gamma R^2 \leqslant 1/4$, then for $u \in [0, \frac{1}{\alpha}+1]$:
\begin{align}
\E \left[  \left\| H^{u/2} \bar{\mu}_t  \right\|^2  \right] \leqslant  4 \sigma^2 \gamma^{1-u} \ \frac{\gamma^{1/\alpha}\tr H^{1/\alpha}}{t^{u-1/\alpha}}.
\end{align}
\end{prop}

\subsection{Proof principle}

We follow closely the proof technique of \cite{bach2013non}, and prove Proposition \ref{prop:generalSGD} by showing it first for a ``semi-stochastic'' recursion, where $z_t \otimes z_t$ is replaced by its expectation (see Lemma \ref{lemma:semi-stoSGD}). We will then compare our general recursion to the semi-stochastic one.

\subsection{Semi-stochastic recursion}
\label{sec:semisto}
\begin{lemma}[Semi-stochastic SGD]
\label{lemma:semi-stoSGD}
Let us consider the following recursion $\mu_t = \left(I - \gamma H \right)\mu_{t-1} + \gamma  \xi_t$ started at $\mu_0 = 0$. Assume that  $\tr H^{1/\alpha} $ is finite, $\E\left[ \xi_t \right] = 0$, $\E\left[ \xi_t \otimes \xi_t\right] \preccurlyeq \sigma^2 H$ and $\gamma H \preccurlyeq I$, then for $u \in [0, \frac{1}{\alpha}+1]$:
\begin{align}
\E \left[  \left\| H^{u/2} \bar{\mu}_t  \right\|^2  \right] \leqslant \sigma^2 \gamma^{1-u} \ \gamma^{1/\alpha}\tr H^{1/\alpha}t^{1/\alpha-u}.
\end{align}
\end{lemma}
\begin{proof}
For $t \geqslant 1$ and $u \in [0, \frac{1}{\alpha}+1]$, using an explicit formula for $\mu_t$ and $\bar{\mu}_t$ (see  \cite{bach2013non} for details), we get:
\begin{align*}
\mu_t &= \left(I - \gamma H \right)\mu_{t-1} + \gamma  \xi_t = \left(I - \gamma H \right)^t\mu_{0} + \gamma \sum_{k = 1}^t \left(I - \gamma H \right)^{t-k} \xi_k \\
\bar{\mu}_t &= \frac{1}{t}\sum_{u = 1}^t \mu_u =  \frac{\gamma}{t}\sum_{u = 1}^t \sum_{k = 1}^u \left(I - \gamma H \right)^{u-k} \xi_k =  \frac{1}{t}\sum_{k = 1}^t H^{-1} \left( I - \left(I - \gamma H \right)^{t-k+1}\right) \xi_k \\
 \E \left[  \left\| H^{u/2} \bar{\mu}_t  \right\|^2  \right] & =  \frac{1}{t^2}\E \sum_{k = 1}^t \tr\left[ \left( I - \left(I - \gamma H \right)^{t-k+1}\right)^2H^{u-2}  \xi_k \otimes \xi_k \right]\\
 & \leqslant  \frac{\sigma^2}{t^2}\sum_{k = 1}^t \tr\left[ \left( I - \left(I - \gamma H \right)^k\right)^2 H^{u-1} \right]
  \mbox{ using } \E\left[ \xi_t \otimes \xi_t\right] \preccurlyeq \sigma^2 H.
\end{align*}
Now, let $(\lambda_i)_{i \in \N^*}$ be the non-increasing sequence of eigenvalues of the operator $H$. We obtain:
\begin{align*}
\E \left[  \left\| H^{u/2} \bar{\mu}_t  \right\|^2  \right] & \leqslant \frac{\sigma^2}{t^2}\sum_{k = 1}^t \sum_{i = 1}^\infty  \left( I - \left(I - \gamma \lambda_i \right)^k\right)^2 \lambda_i^{u - 1}.
\end{align*}
We can now use a simple result\footnote{Indeed, adapting a similar result from~\cite{bach2013non}, on the one hand, $1-(1-\rho )^k\leqslant 1$ implying that $(1-(1-\rho)^k)^{1-1/\alpha+u}\leqslant 1$. On the other hand, $1-(1-\gamma x )^k\leqslant \gamma k x $ implying that $(1-(1-\rho )^k)^{1+1/\alpha-u}\leqslant (  k \rho)^{1+1/\alpha-u}$. Thus by multiplying the two we get $(1-(1-\rho)^k)^2 \leqslant (k \rho )^{1-u + 1/\alpha}$.}  that for any $\rho \in [0,1]$, $ k \geqslant 1$ and $u \in [0, \frac{1}{\alpha}+1]$, we have : $(1-(1-\rho )^k)^2 \leqslant (k \rho )^{1-u + 1/\alpha}$, applied to $\rho = \gamma \lambda_i$. We get, by comparing sums to integrals:
\begin{align*}
\E \left[  \left\| H^{u/2} \bar{\mu}_t  \right\|^2  \right] & \leqslant \frac{\sigma^2}{t^2}\sum_{k = 1}^t \sum_{i = 1}^\infty  \left( I - \left(I - \gamma \lambda_i \right)^k\right)^2 \lambda_i^{u - 1}\\
& \leqslant \frac{\sigma^2}{t^2}\sum_{k = 1}^t \sum_{i = 1}^\infty (k \gamma \lambda_i )^{1-u + 1/\alpha}  \lambda_i^{u - 1} \\
& \leqslant \frac{\sigma^2}{t^2} \gamma^{1-u+1/\alpha}\tr H^{1/\alpha}\sum_{k = 1}^t   k^{1-u+1/\alpha}\\
& \leqslant \frac{\sigma^2}{t^2} \gamma^{1-u+1/\alpha}\tr H^{1/\alpha}\int_{ 1}^t   y^{1-u+1/\alpha} dy \\
& \leqslant \frac{\sigma^2}{t^2} \gamma^{1-u} \ \gamma^{1/\alpha}\tr H^{1/\alpha}\frac{t^{2-u+1/\alpha}}{2-u+1/\alpha} \\
& \leqslant \sigma^2 \gamma^{1-u} \ \gamma^{1/\alpha}\tr H^{1/\alpha}t^{1/\alpha-u},
\end{align*}
which shows the desired result.
\end{proof}

\subsection{Relating the semi-stochastic recursion to the main recursion}
Then, to relate the semi-stochastic recursion with the true one, we use an expansion in the powers of $\gamma$ using recursively the perturbation idea from~\cite{aguech2000perturbation}. 

For $r \geqslant 0$, we define the sequence $(\mu_t^r)_{t \in \N}$, for $t \geqslant 1$, 
\begin{align}
\mu_t^r = (I - \gamma H) \mu_{t-1}^r + \gamma \Xi_t^r,\textrm{ with } \Xi_t^r = 
\begin{cases} 
\ (H - z_t \otimes z_t) \mu_{t-1}^{r-1} \textrm{ if } r \geqslant 1 \\ 
\ \ \Xi_t^0 = \xi_t 
\end{cases}. 
\end{align}
We will show that $\mu_t \simeq \sum_{i = 0}^{\infty} \mu_t^i$. To do so, notice that for $r \geqslant 0$,  $\mu_t - \sum_{i = 0}^{r} \mu_t^i$ follows the recursion:
\begin{align}
\mu_t - \sum_{i = 0}^{r} \mu_t^i = (I - z_t \otimes z_t)\left(\mu_{t-1} - \sum_{i = 0}^{r} \mu_{t-1}^i\right) + \gamma \Xi_t^{r+1},
\end{align}
so that by bounding the covariance operator we can apply a classical SGD result. This is the purpose of the following lemma.
\begin{lemma}[Bound on covariance operator]
\label{lemma:Boundcovariance}
For any $r \geqslant 0$, we have the following inequalities:
\begin{align}
\E\left[ \Xi_t^{r}\otimes \Xi_t^{r}\right] \preccurlyeq \gamma^rR^{2r}\sigma^2H \ \textrm{ and } \ \E\left[ \mu_t^{r}\otimes \mu_t^{r}\right] \preccurlyeq \gamma^{r+1}R^{2r}\sigma^2 I.
\end{align}
\end{lemma}
\begin{proof}
We propose a proof by  induction on $r$. For $r=0$, and $t \geqslant 0$, $\E\left[ \Xi_t^{0}\otimes \Xi_t^{0}\right] = \E\left[ \xi_t\otimes \xi_t\right] \preccurlyeq \sigma^2 H$ by assumption. Moreover,
\begin{align*}
\E\left[ \mu_t^{0}\otimes \mu_t^{0}\right] &= \gamma^2 \sum_{k=1}^{t-1}(I - \gamma H)^{t-k} \E\left[ \Xi_t^{0}\otimes \Xi_t^{0}\right] (I - \gamma H)^{t-k} \preccurlyeq \gamma^2 \sigma^2  \sum_{k=1}^{t-1} (I - \gamma H)^{2(t-k)} H \preccurlyeq \gamma \sigma^2 I.
\end{align*}
Then, for $r \geqslant 1$, 
\begin{align*}
\E\left[ \Xi_t^{r+1}\otimes \Xi_t^{r+1}\right] &\preccurlyeq \E [(H - z_t \otimes z_t)\mu_{t-1}^r \otimes \mu_{t-1}^r (H - z_t \otimes z_t)] \\
&= \E [(H - z_t \otimes z_t)\E[\mu_{t-1}^r \otimes \mu_{t-1}^r] (H - z_t \otimes z_t)] \\
&\preccurlyeq \gamma^{r+1}R^{2r}\sigma^2\E [(H - z_t \otimes z_t)^2]\\
&\preccurlyeq \gamma^{r+1}R^{2r+2}\sigma^2 H.
\end{align*}
And,
\begin{align*}
\E\left[ \mu_t^{r+1}\otimes \mu_t^{r+1}\right] &= \gamma^2 \sum_{k=1}^{t-1}(I - \gamma H)^{t-k} \E\left[ \Xi_t^{r+1}\otimes \Xi_t^{r+1}\right] (I - \gamma H)^{t-k}  \\
&\preccurlyeq \gamma^{r+3} R^{2r+2}\sigma^2  \sum_{k=1}^{t-1} (I - \gamma H)^{2(t-k)} H \preccurlyeq \gamma^{r+2} R^{2r+2}\sigma^2 I,
\end{align*}
which thus shows the lemma by induction.
\end{proof}
To bound $\mu_t - \sum_{i = 0}^{r} \mu_t^i$, we prove a very loose result for the average iterate, that will be sufficient for our purpose.
\begin{lemma}[Bounding SGD recursion]
\label{lemma:BoundSGD}
Let us consider the following recursion $\mu_t = \left(I - \gamma z_t \otimes z_t \right)\mu_{t-1} + \gamma  \xi_t$ starting at $\mu_0 = 0$. Assume that $\E [z_t \otimes z_t ] = H$, $\E\left[ \xi_t \right] = 0$, $\|x_t\|^2 \leqslant R^2 $, $\E\left[ \xi_t \otimes \xi_t\right] \preccurlyeq \sigma^2 H$ and $\gamma R^2 < I$, then for $u \in [0, \frac{1}{\alpha}+1]$:
\begin{align}
\E \left[  \left\| H^{u/2} \bar{\mu}_t  \right\|^2  \right] \leqslant \sigma^2 \gamma^2 R^u \tr H \ t.
\end{align}
\end{lemma}
\begin{proof}
Let us define the operators for $j \leqslant i$  : $M^i_{j} = (I - \gamma z_{i(i)}\otimes z_{i(i)})\cdots(I - \gamma z_{i(j)}\otimes z_{i(j)})$ and $M^i_{i+1} = I$. Since $\mu_0 = 0$, note that we have we have, $\mu_i = \gamma \sum_{k = 1}^i M^i_{k+1} \xi_k $. Hence, for $i \geqslant 1$,
\begin{align*}
\E \left\| H^{u/2} \mu_i  \right\|^2 &= \gamma^2 \E \sum_{k,j} \langle M^i_{j+1} \xi_j , H^{u}M^i_{k+1} \xi_k \rangle \\
&= \gamma^2 \E \sum_{k = 1}^i \langle M^i_{k+1} \xi_k , H^{u}M^i_{k+1} \xi_k \rangle \\
&= \gamma^2 \tr\ \left(\E \left[\sum_{k = 1}^i  {M^i_{k+1}}^* H^{u} M^i_{k+1} \xi_k \otimes \xi_k \right]\right)  \leqslant \sigma^2 \gamma^2 \E \left[\sum_{k = 1}^i \tr \left( {M^i_{k+1}}^* H^{u} M^i_{k+1} H \right) \right] \\
& \leqslant \sigma^2 \gamma^2 R^u i\ \tr H,
\end{align*}
because $\tr \left( {M^i_{k+1}}^* H^{u} M^i_{k+1} H \right) \leqslant R^u \tr H$. Then,
\begin{align*}
\E \left\| H^{u/2} \bar{\mu}_t  \right\|^2 &= \frac{1}{t^2} \sum_{i,j} \langle H^{u/2} \mu_i,H^{u/2} \mu_j  \rangle \\
&\leqslant \frac{1}{t^2} \E\left( \sum_{i = 1}^t \left\| H^{u/2} \mu_i\right\| \right)^2\leqslant \frac{1}{t}  \sum_{i = 1}^t \E\left\| H^{u/2} \mu_i\right\|^2 \leqslant \sigma^2 \gamma^2 R^u \tr H \ t,
\end{align*}
which finishes the proof of Lemma~\ref{lemma:BoundSGD}.
\end{proof}

\subsection{Final steps of the proof}

We have now all the material to conclude. Indeed by the triangular inequality:
\begin{align*}
\left(\E  \left\| H^{u/2} \bar{\mu}_t  \right\|^2  \right)^{1/2} \leqslant  \sum_{i=1}^r \left(\underbrace{\E  \left\| H^{u/2} \bar{\mu}_t^i  \right\|^2 }_{\textrm{Lemma}\ \ref{lemma:semi-stoSGD}}  \right)^{1/2} + \left(\underbrace{\E   \left\| H^{u/2} \left(\bar{\mu}_t - \sum_{i=1}^r \bar{\mu}_t^i \right) \right\|^2}_{\textrm{Lemma} \ \ref{lemma:BoundSGD}} \right)^{1/2}.
\end{align*}
With Lemma \ref{lemma:Boundcovariance}, we have all the bounds on the covariance of the noise, so that:
\begin{align*}
\left(\E  \left\| H^{u/2} \bar{\mu}_t  \right\|^2  \right)^{1/2} &\leqslant  \sum_{i=1}^r \left( \gamma^i R^{2i}\sigma^2 \gamma^{1-u} \ \gamma^{1/\alpha}\tr H^{1/\alpha}t^{1/\alpha-u} \right)^{1/2} + \left( \gamma^{r+2}R^{2r + u}\tr H\ t \right)^{1/2} \\
&\leqslant (\sigma^2 \gamma^{1-u} \ \gamma^{1/\alpha}\tr H^{1/\alpha}t^{1/\alpha-u})^{1/2} \sum_{i=1}^r \left( \gamma R^{2}\right)^{i/2}  + \left( \gamma^{r+2}R^{2r + u}\tr H\ t \right)^{1/2}.
\end{align*}
Now we make $r$ go to infinity and we obtain:
\begin{align*}
\left(\E  \left\| H^{u/2} \bar{\mu}_t  \right\|^2  \right)^{1/2} &\leqslant (\sigma^2 \gamma^{1-u} \ \gamma^{1/\alpha}\tr H^{1/\alpha}t^{1/\alpha-u})^{1/2}  \frac{1}{1-\sqrt{\gamma R^2}} + \underbrace{\left( \gamma^{r+2}R^{2r + u}\tr H\ t \right)^{1/2}}_{\underset{r \rightarrow \infty}{\longrightarrow }0}
\end{align*}
Hence with $\gamma R^2 \leqslant 1/4$,
\begin{align*}
\E  \left\| H^{u/2} \bar{\mu}_t  \right\|^2 &\leqslant  4 \sigma^2 \gamma^{1-u} \ \gamma^{1/\alpha}\tr H^{1/\alpha}t^{1/\alpha-u},
\end{align*}
 which finishes to prove Proposition~\ref{prop:generalSGD}.

\section{Proof sketch for Theorem~\ref{thm:main_result}}
\label{sec:proof}

 We consider the batch gradient descent recursion, started from $\eta_0 = 0$, with the same step-size:
 \[
 \eta_t  = \eta_{t-1} + \frac{\gamma}{n} \sum_{i=1}^n\big( y_{i} - \langle \eta_{t-1}, \Phi(x_{i}) \rangle_\h \big) \Phi(x_{i}),
 \]
 as well as its averaged version $\bar{\eta}_t = \frac{1}{t}\sum_{i=0}^t \eta_i$. We obtain a recursion for $\theta_t - \eta_t$, with the initialization $\theta_0 - \eta_0 = 0$, as follows:
 \[
 \theta_t - \eta_t
 = \big[  I -  \Phi(x_{i(u)}) \otimes_\h \Phi(x_{i(u)}) \big] ( \theta_{t-1} - \eta_{t-1})  +  \gamma \xi_t^1 + \gamma \xi_t^2,
  \]
  with $\xi_t^1 =   y_{i(u)} \Phi(x_{i(u)}) - \frac{1}{n} 
 \sum_{i=1}^n y_i \Phi(x_{i})$ and $\xi_t^2 = \big[   \Phi(x_{i(u)}) \otimes_\h \Phi(x_{i(u)}) - \frac{1}{n} \sum_{i=1}^n  \Phi(x_{i}) \otimes_\h \Phi(x_{i}) \big]\eta_{t-1}$. We decompose the performance $F(\theta_t)$ in two parts, one analyzing the performance of batch gradient descent, one analyzing the deviation $\theta_t - \eta_t$, using
 \[
\E  F(\bar{\theta}_t) - F(\theta_\ast)
\leqslant 2  \E \big[ \| \Sigma^{1/2} ( \theta_t - \eta_t )\|_\h^2 \big] + 2 \big[ \E  F(\bar{\eta}_t) - F(\theta_\ast) \big].
 \]
 We denote by $\hat{\Sigma}_n = \frac{1}{n} \sum_{i=1}^n \Phi(x_i) \otimes \Phi(x_i)$ the empirical second-order moment.
 
 \paragraph{\bfseries Deviation $\theta_t - \eta_t$.}

 Denoting by $\G$ the $\sigma$-field generated by the data and by $\mathcal{F}_t$ the $\sigma$-field generated by $i(1),\dots,i(t)$, then, we have
  $ \E ( \xi_t^1 | \G, \F_{t-1})  =  \E ( \xi_t^2 | \G, \F_{t-1}) = 0$, thus we can apply results for averaged SGD (see Proposition \ref{prop:generalSGD} of the Appendix) to get the following lemma.
   \begin{lemma}
   \label{lemma:deviationsgd}
   For any $t \geqslant 1$, if $\E \big[ (  \xi_t^1 +   \xi_t^2) \otimes_\h (   \xi_t^1 +   \xi_t^2 )  | \G \big] \preccurlyeq \tau^2 \hat{\Sigma}_n$, and $4 \gamma R^2 =  1$, under Assumptions~\ref{asm:iid},~\ref{asm:bounded},~\ref{asm:capacity_condition},
  \begin{align}
  \E \big[ \| \hat{\Sigma}_n^{1/2} (  \bar\theta_t - \bar\eta_t ) \|_\h^2 | \G \big]
  \leqslant  \frac{8 \tau^2 \gamma^{1/\alpha}\tr\ \hat{\Sigma}_n^{1/\alpha}}{t^{1-1/\alpha}}.
  \end{align}
  \end{lemma}
  In order to obtain the bound, we need to bound $\tau^2$ (which is dependent on $\G$) and go from a bound with the empirical covariance matrix $\hat{\Sigma}_n$ to bounds with the population covariance matrix $\Sigma$.

We have
\[
\E \big[    \xi_t^1   \otimes_\h     \xi_t^1   | \G \big] \preccurlyeq_\h
\E \big[    y_{i(u)}^2   \Phi(x_{i(u)}) \otimes_\h \Phi(x_{i(u)})  | \G \big] \preccurlyeq_\h \| y\|_\infty^2 \hat{\Sigma}_n 
\preccurlyeq_\h ( \sigma + \sup_{x \in \X} \langle \theta_\ast, \Phi(x) \rangle_\h  ) ^2 \hat{\Sigma}_n 
\]
\[
\E \big[    \xi_t^2   \otimes     \xi_t^2   | \G \big] \preccurlyeq_\h
\E \big[    \langle \eta_{t-1}, \Phi(x_{i(u)}) \rangle^2  \Phi(x_{i(u)}) \otimes_\h \Phi(x_{i(u)})  | \G \big] \preccurlyeq_\h   \sup_{t \in \{0,\dots,T-1\}} \sup_{x \in \X} \langle \eta_{t} , \Phi(x) \rangle_\h  ) ^2 \hat{\Sigma}_n 
\]
Therefore $\tau^2 = 2 M^2
+ 2 \sup_{t \in \{0,\dots,T-1\}} \sup_{x \in \X} \langle \eta_{t} , \Phi(x) \rangle_\h  ^2$ or using Assumption \ref{asm:hyp-space}
$\tau^2 = 2 M^2
+ 2 \sup_{t \in \{0,\dots,T-1\}} R^{2\mu} \kappa_\mu^2  \| \Sigma^{1/2-\mu/2} \eta_t\|_\h^2$.
 
In the proof, we rely on an event (that depend on $\G$) where $\Cn$ is close to $\Sigma$. This leads to the the following Lemma that bounds the deviation $\bar{\theta}_t - \bar{\eta}_t$.

   \begin{lemma}
   \label{lem:SGDrecursion}
   For any $t \geqslant 1$, $4 \gamma R^2 =  1$, under Assumptions~\ref{asm:iid},~\ref{asm:bounded},~\ref{asm:capacity_condition}, 
  \begin{align}
\E \big[ \| \Sigma^{1/2} (  \bar\theta_t - \bar\eta_t ) \|_\h^2 \big] \leqslant 16 \tau^2_\infty \left[ R^{-2/\alpha}\tr\ \Sigma^{1/\alpha}t^{1/\alpha}\left( \frac{1}{t} + \left( \frac{4}{\mu}\frac{\log n}{ n} \right)^{1/\mu}   \right) + 1 \right].
  \end{align}
  \end{lemma}

We make the following remark on the bound. 

\begin{rmk}
Note that as defined in the proof $\tau_\infty$ may diverge in some cases as \begin{align*}
\tau^2_\infty = 
\begin{cases}
O(1) \hspace{2.38cm} \textrm{when } \mu \leqslant 2r,\\
O(n^{\mu - 2r}) \hspace{1.5cm}  \textrm{when }\ 2 r \leqslant \mu \leqslant 2 r + 1/\alpha ,\\
O( n^{1 - 2r/\mu})\hspace{1.21cm}  \textrm{when } \mu \geqslant 2 r + 1/\alpha ,
\end{cases}
\end{align*}
with $O(\cdot)$ are defined explicitly in the proof.
\end{rmk}
\paragraph{\bfseries Convergence of batch gradient descent.}
The main result is summed up in the following lemma, with $t = O(n^{1/\mu})$ and $t \geqslant n$.
\begin{lemma}
\label{lem:BGD_result}
Let $t > 1$, under Assumptions~\ref{asm:iid}, \ref{asm:bounded}, \ref{asm:hyp-space},~\ref{asm:capacity_condition},~\ref{asm:source_condition},~\ref{asm:reg-Pfrho}, when, with $4 \gamma R^2=1$,
\eqal{
	t ~~=~ \begin{cases}
		\Theta( n^{\alpha/\left(2r\alpha + 1 \right)}) & 2r\alpha + 1 > \mu\alpha \\
		\Theta( n^{1/\mu}~(\log n)^{\frac{1}{\mu}}) & 2r\alpha + 1 \leqslant \mu\alpha.
	\end{cases}
}
then,
\eqal{
\E  F(\bar{\eta}_t) - F(\theta_\ast)  \leqslant  \begin{cases}
O( n^{-2r\alpha/\left(2r\alpha + 1 \right)}) & 2r\alpha + 1 > \mu\alpha \\
O( n^{-2r/\mu}) & 2r\alpha + 1  \leqslant \mu\alpha
\end{cases}
}
with $O(\cdot)$ are defined explicitly in the proof. \end{lemma}

\begin{rmk}
In all cases, we can notice that the speed of convergence of Lemma \ref{lem:BGD_result} are slower that the ones in Lemma \ref{lem:SGDrecursion}, hence, the convergence of the gradient descent controls the rates of convergence of the algorithm.
\end{rmk}

\section{Bounding the deviation between SGD and batch gradient descent}

\label{sec:SGD-BGD}

In this section, following the proof sketch from \mysec{proof}, we provide a bound on the deviation
$\theta_t - \eta_t$.
In all the following let us denote $\mu_t = \theta_t - \eta_t$ that deviation between the stochastic gradient descent recursion and the batch gradient descent recursion.

\subsection{Proof of Lemma \ref{lem:SGDrecursion}}
 
We need to (a) go from $\hat{\Sigma}_n$ to $\Sigma$ in the result of Lemma~\ref{lemma:deviationsgd} and (b) to have a bound on $\tau$. To prove this result we are going to need the two following lemmas:
\begin{lemma}
\label{lem:bounding-beta2}
Let $\lambda > 0$, $\delta \in (0,1]$. Under Assumption~\ref{asm:hyp-space},  
when $n \geqslant 11(1+\kappa^2_\mu R^{2\mu}\gamma^\mu  t^{\mu}) \log \frac{8R^2}{\lambda\delta},$
  the following holds with probability $1-\delta$,
\begin{align}
\label{eq:sigma_hat_sigma}
\left\| (\Sigma + \lambda I)^{1/2} (\hat{\Sigma}_n + \lambda I)^{-1/2} \right\|^2 \leqslant 2.
\end{align}
\end{lemma}
\begin{proof} 
This Lemma is proven and stated lately in Lemma \ref{lem:bounding-beta} in \mysec{probbounds}. We recalled it here for the sake of clarity.
\end{proof}
\begin{lemma}
\label{lem:bounding-tau}
Let $\lambda > 0$, $\delta \in (0,1]$. Under Assumption~\ref{asm:hyp-space}, for $t = O\left(\frac{1}{n^{1/\mu}}\right)$
 then the following holds with probability $1-\delta$,
\begin{align}
\label{eq:bounding-tau}
\tau^2 \leqslant \tau^2_\infty \quad {\textrm and}\quad \tau^2_\infty  = 
\begin{cases}
O(1)  ,\ \textrm{when } \mu\leqslant 2r,\\
O\left( n^{\mu - 2r} \right),\ \textrm{when } 2r \leqslant \mu \leqslant 2 r  + 1/\alpha,\\
O\left( n^{1 - 2r/\mu} \right)\ \textrm{when } \mu \geqslant 2 r + 1/\alpha,
\end{cases}
\end{align}
where the $O(\cdot)$-notation depend only on the parameters of the problem (and is independent of $n$ and~$t$).
\end{lemma}
\begin{proof}
This Lemma is a direct implication of Corollary \ref{cor:bound-gla-Linfty} in \mysec{probbounds}. We recalled it here for the sake of clarity.
\end{proof}
Note that we can take $\lambda^\delta_n = \left( \frac{\log \frac{n}{\delta}}{n} \right)^{1/\mu}$ so that Lemma \ref{lem:bounding-beta2} result holds. 
Now we are ready to prove Lemma \ref{lem:SGDrecursion}.

\begin{proof}[Proof of Lemma \ref{lem:SGDrecursion}]
Let $A_{\delta_a}$ be the set for which inequality \eqref{eq:sigma_hat_sigma} holds and let $B_{\delta_b}$ be the set for which inequality \eqref{eq:bounding-tau} holds. Note that $\mathbb{P}(A_{\delta_a}^c) = \delta_a$ and $\mathbb{P}(B_{\delta_b}^c) = \delta_b$. We use the following decomposition:
 \[
 \E \left\| \Sigma^{1/2} \bar{\mu}_t  \right\|^2 \leqslant \E \left[\left\| \Sigma^{1/2} \bar{\mu}_t  \right\|^2 \mathbf{1}_{A_{\delta_a}\cap B_{\delta_b}}\right] +  \E \left[\left\| \Sigma^{1/2} \bar{\mu}_t  \right\|^2 \mathbf{1}_{A^c_{\delta_a}}\right] +  \E \left[\left\| \Sigma^{1/2} \bar{\mu}_t  \right\|^2 \mathbf{1}_{ B^c_{\delta_b}}\right].
 \]
First, let us bound roughly $\|\bar{\mu}_t\|^2$. 

First, for $i \geqslant 1$, $\|\mu_i\|^2 \leqslant \gamma^2 \left( \sum_{i=1}^t{\|\xi^1_i\|+\|\xi^2_i\|} \right)^2 \leqslant 16 R^2 \gamma^2 \tau^2 t^2 $, so that $\|\bar{\mu}_t\|^2 \leqslant \frac{1}{t} \sum_{i=1}^t\|\mu_i\|^2 \leqslant 16 R^2 \gamma^2 \tau^2 t^2$. We can bound similarly $\tau^2  \leqslant 4   M^2 \gamma^2R^4  t^2$, so that 
$\|\bar{\mu}_t\|^2   
 \leqslant 64 R^2 M^2 \gamma^4 t^4 $. Thus, for the second term:
\[ \E \left[\left\| \Sigma^{1/2} \bar{\mu}_t  \right\|^2 \mathbf{1}_{A_{\delta_a}^c}\right] \leqslant 64 R^8 M^2 \gamma^4 t^4 \E \mathbf{1}_{A_{\delta_a}^c} \leqslant 64 R^8 M^2 \gamma^4 t^4 \delta_a, \]
and for the third term:
\[
\E \left[\left\| \Sigma^{1/2} \bar{\mu}_t  \right\|^2 \mathbf{1}_{B_{\delta_b}^c}\right] \leqslant 64 R^8 M^2 \gamma^4 t^4 \E \mathbf{1}_{B_{\delta_b}^c} \leqslant 64 R^8 M^2 \gamma^4 t^4 \delta_b.
\] 
And on for the first term,
\begin{align*}
\E \left[ \left\| \Sigma^{1/2} \bar{\mu}_t  \right\|^2\mathbf{1}_{A_{\delta_a}\cap B_{\delta_b}} \right]  &\leqslant \E\left[\left\| \Sigma^{1/2} (\Sigma + \lambda^\delta_n I)^{-1/2} \right\|^2 \left\| (\Sigma + \lambda^\delta_n I)^{1/2} (\hat{\Sigma}_n + \lambda^\delta_n I)^{-1/2} \right\|^2 \right. \\
&\left. \hspace{4.6cm} \left\| (\hat{\Sigma}_n + \lambda^\delta_n I)^{1/2} \bar{\mu}_t  \right\|^2 \mathbf{1}_{A_{\delta_a}\cap B_{\delta_b}}\ |\ \G \right] \\
&\leqslant 2 \E \left[ \left\| (\hat{\Sigma}_n + \lambda^\delta_n I)^{1/2} \bar{\mu}_t  \right\|^2 \ |\ \G \right] \\
& = 2 \E \left[\left\| \hat{\Sigma}_n^{1/2} \bar{\mu}_t  \right\|^2\ |\ \G \right] + 2\lambda^\delta_n\E \left[ \left\| \bar{\mu}_t  \right\|^2 \ |\ \G \right] \\
&\leqslant 16 \tau_\infty^2 \ \frac{\gamma^{1/\alpha}\E \left[\tr\ \hat{\Sigma}_n^{1/\alpha} \right]}{t^{1-1/\alpha}} + 8 \lambda_n^\delta \tau_\infty^2  \ \gamma^{1/\alpha}\E \left[\tr\ \hat{\Sigma}_n^{1/\alpha} \right]t^{1/\alpha},
\end{align*}
using Proposition~\ref{prop:generalSGD} twice with $u=1$ for the left term and $u=1$ for the right one.

As $x \rightarrow x^{1/\alpha}$ is a concave function, we can apply Jensen's inequality to have : \[\E\left[ \tr (\hat{\Sigma}_n^{1/\alpha})\right] \leqslant \tr \Sigma^{1/\alpha},\] so that:
\begin{align*}
\E \left[ \left\| \Sigma^{1/2} \bar{\mu}_t  \right\|^2\mathbf{1}_{A_{\delta_a}\cap B_{\delta_b}} \right] &\leqslant 16 \tau_\infty^2 \ \frac{\gamma^{1/\alpha}\tr\ \Sigma^{1/\alpha}}{t^{1-1/\alpha}} + 8 \lambda_n^\delta \tau_\infty^2 \gamma\ \gamma^{1/\alpha}\tr\ \Sigma^{1/\alpha}t^{1/\alpha} \\
&\leqslant 16 \tau_\infty^2 \gamma^{1/\alpha}\tr\ \Sigma^{1/\alpha}t^{1/\alpha}\left( \frac{1}{t} + \lambda_n^\delta   \right). 
\end{align*}
Now, we take $\delta_a=\delta_b = \frac{\tau_\infty^2}{4 M^2 R^8\gamma^4 t^4}$ and this concludes the proof of Lemma \ref{lem:SGDrecursion}, with the bound:
\[
\E \left\| \Sigma^{1/2} \bar{\mu}_t  \right\|^2
\leqslant 
16 \tau_\infty^2 \gamma^{1/\alpha}\tr\ \Sigma^{1/\alpha}t^{1/\alpha}\left( \frac{1}{t} + 
\Big(
\frac{ 2 + 2 \log M + 4 \log (\gamma R^2) + 4 \log t}{n}
\Big)^{1/\mu} \right). 
\]
\end{proof}
\section{Convergence of batch gradient descent}
\label{sec:ales}
In this section we prove the convergence of averaged batch gradient descent to the target function. In particular, since the proof technique is valid for the wider class of algorithms known as spectral filters \cite{gerfo2008spectral,lin2018optimal}, we will do the proof for a generic spectral filter (in Lemma~\ref{lem:filter_average}, Sect.~\ref{app:notations} we prove that averaged batch gradient descent is a spectral filter).

In Section \ref{app:notations} we provide the required notation and additional definitions. In Section~\ref{app:analytical_decomposition}, in particular in Theorem~\ref{app:analytical_decomposition} we perform an analytical decomposition of the excess risk of the averaged batch gradient descent, in terms of basic quantities that will be controlled in expectation (or probability) in the next sections.
In Section \ref{sec:probbounds} the various quantites obtained by the analytical decomposition are controlled, in particular, Corollary~\ref{cor:bound-gla-Linfty} controls the $L^\infty$ norm of the averaged batch gradient descent algorithm.
Finally in Section \ref{app:main_result}, the main result, Theorem \ref{thm:final-filters} controlling in expectation of the excess risk of the averaged batch gradient descent estimator is provided. In Corollary~\ref{cor:final_bound_lambda}, a version of the result of Theorem \ref{thm:final-filters} is given, with explicit rates for the regularization parameters and of the excess risk.

\subsection{Notations}
\label{app:notations}
In this subsection, we study the convergence of batch gradient descent. For the sake of clarity we consider the RKHS framework (which includes the finite-dimensional case). We will thus consider elements of $\h$ that are naturally embedded in $L_2(d \rho_\X)$ by the operator $\S$ from $\h$ to $L_2(d \rho_\X)$ and such that: $ (\S g)(x) = \langle g ,  K_x \rangle$, where we have $\Phi(x)  = K_x = K(\cdot,x)$ where $K: \X \to \X \to \R$ is the kernel. We recall the recursion for $\eta_t $ in the case of an RKHS feature space with kernel $K$:
 \[
 \eta_t  = \eta_{t-1} + \frac{\gamma}{n} \sum_{i=1}^n\big( y_{i} - \langle \eta_{t-1}, K_{x_{i}} \rangle_\h \big) K_{x_{i}},
 \]
Let us begin with some notations. In the following we will often use the letter $g$ to denote vectors of $\h$, hence, $Sg$ will denote functions of $\Ltwo$. We also define the following operators (we may also use their adjoints, denoted with a $*$):
\begin{itemize}
\item The operator $\Sn$ from $\h$ to $\R^n$, $ \Sn g = \frac{1}{\sqrt{n}}(g(x_1),\dots g(x_n))$.
\item The operators from $\h$ to $\h$, $\C$ and $\Cn$, defined respectively as $\C = \E \left[K_x \otimes K_x \right] = \int_{\X} K_x \otimes K_x d\rho_\X$ and $\Cn = \frac{1}{n} \sum_{i=1}^n K_{x_i} \otimes K_{x_i} $. Note that $\Sigma$ is the covariance operator.
\item The operator $\L: \Ltwo \to \Ltwo$ is defined by
\[(\L f)(x) = \int_\X K(x,z) f(z) d\rhox(x), \quad \forall f \in \Ltwo.\]

Moreover denote by ${\mathcal {N}(\la)}$ the so called {\em effective dimension} of the learning problem, that is defined as
\[{\mathcal {N}}(\la) = \tr(\L (\L+\la I)^{-1}),\]
for $\la > 0$.
Recall that by Assumption \ref{asm:capacity_condition}, there exists $\alpha \geqslant 1$ and $Q > 0$ such that
\[{\mathcal {N}}(\la) \leqslant Q \la^{-1/\alpha}, \quad \forall \la > 0.\]
We can take $Q = \tr \Sigma^{1/\alpha}$.

\item $P: \Ltwo \to \Ltwo$ projection operator on $\h$ for the $\Ltwo$ norm s.t. ${\rm ran} P = {\rm ran} \S$.
\end{itemize}
Denote by $\frho$ the function so that $\frho(x) = \expect{y|x} ~\in \Ltwo$ the minimizer of the expected risk, defined by $F(f) = \int_{X\times \R} (f(x) - y)^2 d\rho(x,y).$

\vspace{0.2cm}
\begin{rmk}[On Assumption~\ref{asm:source_condition}]\label{rmk:asm-source}
	With the notation above, we express assumption \ref{asm:source_condition}, more formally, w.r.t. Hilbert spaces with infinite dimensions, as follows. 
	There exists $r \in [0,1]$ and $\phi \in \Ltwo$, such that \[P\frho = \L^r \phi.\] 
\end{rmk}

\begin{enumerate}[label={\bfseries(A\arabic*)},ref=(A\arabic*), leftmargin=*]
\setcounter{enumi}{\value{savenum}}
\item \label{asm:reg-Pfrho} \hspace*{.2cm} \emph{Let $q \in [1, \infty]$ be such that $\|\frho - P\frho\|_{L^{2q}(\X,\rhox)} < \infty.$}
\setcounter{savenum}{\value{enumi}}
\end{enumerate}

The assumption above is always true for $q = 1$, moreover when the kernel is universal it is true even for $q = \infty$. Moreover if $r \geqslant 1/2$ then it is true for $q = \infty$. Note that we make the calculation in this Appendix for a general $q \in [1, \infty]$, but we presented the results for $q = \infty$ in the main paper. The following proposition relates the excess risk to a certain norm.

\begin{prop}\label{prop:excess risk}
When $\widehat{g} \in \h$,
\[F(\widehat{g}) - \inf_{g \in \h} F(g) = \|\S\widehat{g} - P\frho\|_\Ltwo^2.\]
\end{prop} 

We introduce the following function $g_\la \in \h$ that will be useful in the rest of the paper $g_\la = (\C+\la I)^{-1}\S^* \frho.$

We introduce the estimators of the form, for $\la > 0$,
\[ \widehat{g}_\la = q_\la(\Cn) \Sn^* \yn,\]
where $q_\la: \R_+ \to \R_+$ is a function called {\em filter}, that essentially approximates $x^{-1}$ with the approximation controlled by $\la$. Denote moreover with $r_\la$ the function $r_\la(x) = 1 - x q_\la(x)$. The following definition precises the form of the filters we want to analyze. We then prove in Lemma \ref{lem:filter_average} that our estimator corresponds to such a filter.
\begin{defi}[Spectral filters]
Let $q_\la:\R_+ \to \R_+$ be a function parametrized by $\la > 0$. $q_\la$ is called a {\em filter} when there exists $c_q > 0$ for which
\[\la q_\la(x) \leqslant c_q, \quad  r_\la(x) x^u \leqslant c_q \la^u, \quad \forall x > 0, \la > 0, u \in [0,1].\]
\end{defi}
We now justify that we study estimators of the form $\widehat{g}_\la = q_\la(\Cn) \Sn^* \yn$ with the following lemma. Indeed, we show that the average of batch gradient descent can be represented as a filter estimator, $\widehat{g}_\la$, for $\lambda = 1/(\gamma t)$.
\begin{lemma}
\label{lem:filter_average}
For $t > 1$, $\lambda = 1/(\gamma t)$, $\bar{\eta}_t = \widehat{g}_\la$, with respect to the filter, $q^\eta (x) = \left( 1-\frac{1-(1-\gamma x )^t}{\gamma t x} \right) \frac{1}{x} $.
\end{lemma}
\begin{proof}
Indeed, for $t > 1$,
\begin{align*}
\eta_t  &= \eta_{t-1} + \frac{\gamma}{n} \sum_{i=1}^n\big( y_{i} - \langle \eta_{t-1}, K_{x_{i}} \rangle_\h \big) K_{x_{i}} \\
&= \eta_{t-1} + \gamma (\Sn^* \yn - \Cn \eta_{t-1})  \\
&= (I - \gamma \Cn)\eta_{t-1} + \gamma \Sn^* \yn   \\
&= \gamma \sum_{k = 0}^{t-1} (I - \gamma \Cn)^k \Sn^* \yn  =  \left[ I - (I - \gamma \Cn)^t\right] \Cn^{-1} \Sn^* \yn ,
\end{align*}
leading to 
\[ \bar{\eta}_t = \frac{1}{t} \sum_{i=0}^t \eta_i = q^\eta \left(\Cn\right) \Sn^* \yn   .\]
Now, we prove that $q$ has the properties of a filter. 
First, for $t > 1$,  $\frac{1}{\gamma t}q^\eta (x) = \left( 1-\frac{1-(1-\gamma x )^t}{\gamma t x} \right) \frac{1}{\gamma t x} $ is a decreasing function so that $\frac{1}{\gamma t}q^\eta (x) \leqslant \frac{1}{\gamma t}q^\eta (0) \leqslant 1$. 
Second for $u \in [0,1]$, $x^u(1 - x q^\eta(x)) = \frac{1-(1-\gamma x)^t}{\gamma t x} x^u$. As used in  \mysec{semisto}, $1-(1-\gamma x)^t \leqslant (\gamma t x)^{1-u}$, so that, $r^\eta(x)x^u \leqslant \frac{(\gamma t x)^{1-u}}{\gamma t x} x^u = \frac{1}{(\gamma t)^u}$, this concludes the proof that $q^\eta$ is indeed a filter.
\end{proof}

\subsection{Analytical decomposition}
\label{app:analytical_decomposition}

\begin{lemma}\label{lm:variance-12}
Let $\la > 0$ and $s \in (0,1/2]$.
Under Assumption~\ref{asm:source_condition} (see Rem.~\ref{rmk:asm-source}), the following holds
\[\|\L^{-s}\S(\widehat{g}_\la - g_\la)\|_{\Ltwo} \leqslant 2 \la^{-s} \beta^2 c_q \|\Cl^{-1/2}(\Sn^*\yn -  \Cn g_\la)\|_\h  + 2\beta c_q \|\phi\|_\Ltwo \la^{r-s}, \]
where $\beta := \|\Cl^{1/2} \Cnl^{-1/2}\|$.
\end{lemma}
\begin{proof}
By Prop.~\ref{prop:excess risk}, we can characterize the excess risk of $\widehat{g}_\la$ in terms of the $\Ltwo$ squared norm of $S\widehat{g}_\la - P\frho$. In this paper, simplifying the analysis of \cite{lin2018optimal}, we perform the following decomposition
\eqals{
	\L^{-s}\S( \widehat{g}_\la - g_\la) &= \L^{-s}\S\widehat{g}_\la - \L^{-s}\S q_\la(\Cn) \Cn g_\la\\
	& + ~~ \L^{-s}\S q_\la(\Cn) \Cn g_\la - \L^{-s}\S g_\la.
}
{\bf Upper bound for the first term.} By using the definition of $\widehat{g}_\la$ and multiplying and dividing by  $\Cl^{1/2}$, we have that
\eqals{
	\L^{-s}\S\widehat{g}_\la - \L^{-s}\S q_\la(\Cn) \Cn g_\la & =  \L^{-s}\S q_\la(\Cn) (\Sn^*\yn -  \Cn g_\la) \\
	& = \L^{-s}\S q_\la(\Cn) \Cl^{1/2} ~ \Cl^{-1/2}(\Sn^*\yn -  \Cn g_\la),
}
from which
\[ \|\L^{-s}\S(\widehat{g}_\la - q_\la(\Cn) \Cn g_\la)\|_\Ltwo \leqslant \|\L^{-s}\S q_\la(\Cn) \Cl^{1/2}\| ~ \|\Cl^{-1/2}(\Sn^*\yn -  \Cn g_\la)\|_\h.
\]
{\bf Upper bound for the second term.} By definition of $r_\la(x) = 1 - x q_\la(x)$ and $g_\la = \Cl^{-1}\S^*\frho$,
\eqals{
	\L^{-s}\S q_\la(\Cn) \Cn g_\la - \L^{-s}\S g_\la &= \L^{-s}\S (q_\la(\Cn) \Cn - I) g_\la \\
	& = - \L^{-s}\S r_\la(\Cn)~ \Cl^{-(1/2-r)} ~ \Cl^{-1/2 - r} \S^* \L^{r} ~ \phi,
}
where in the last step we used the fact that $\S^*\frho = \S^*P\frho = \S^*\L^r \phi$, by Asm.~\ref{asm:source_condition} (see Rem.~\ref{rmk:asm-source}). Then
\eqals{
	\|\L^{-s}\S(q_\la(\Cn) \Cn  - I)g_\la)\|_\Ltwo &\leqslant \|\L^{-s}\S r_\la(\Cn)\|\|\Cl^{-(1/2-r)}\| \|\Cl^{-1/2 - r} \S^* \L^{r}\|\|\phi\|_\Ltwo \\
	& \leqslant \la^{-(1/2-r)}\|\L^{-s}\S r_\la(\Cn)\|\|\phi\|_\Ltwo,
}
where the last step is due to the fact that $\|\Cl^{-(1/2-r)}\| \leqslant \la^{-(1/2-r)}$ and that $\S^* \L^{2r} \S = \S^*(\S\S^*)^{2r} \S = (\S^*\S)^{2r} \S^*\S = \C^{1+2r}$ from which
\eqal{\label{eq:ClaSL}
	\|\Cl^{-1/2 - r} \S^* \L^{r}\|^2 = \|\Cl^{-1/2 - r} \S^* \L^{2r} \S \Cl^{-1/2 - r}\| = \|\Cl^{-1/2 - r} \C^{1+2r} \Cl^{-1/2 - r}\| \leqslant 1.
}
{\bf Additional decompositions.}
We further bound $\|\L^{-s}\S r_\la(\Cn) \|$ and $\|\L^{-s}\S q_\la(\Cn) \Cl^{1/2}\|$. For the first, by the identity
$\L^{-s} \S r_\la(\Cn) = \L^{-s}\S \Cnl^{-1/2} \Cnl^{1/2}  r_\la(\Cn)$, we have
\[\|\L^{-s}\S r_\la(\Cn)\| = \|\L^{-s}\S \Cnl^{-1/2}\| \|\Cnl^{1/2}  r_\la(\Cn)\|,\]
where
\[\|\Cnl^{1/2}  r_\la(\Cn)\| = \sup_{\sigma \in \sigma(\Cn)} (\sigma + \la)^{1/2} r_\la(\sigma) \leqslant \sup_{\sigma \geqslant 0} (\sigma + \la)^{1/2} r_\la(\sigma) \leqslant 2c_q \la^{1/2}.\]
Similarly, by using the identity 
\[\L^{-s}\S q_\la(\Cn) \Cl^{1/2} = \L^{-s}\S \Cnl^{-1/2}~ \Cnl^{1/2} q_\la(\Cn)\Cnl^{1/2}~\Cnl^{-1/2}\Cl^{1/2},\]
we have
\[\|\L^{-s}\S q_\la(\Cn) \Cl^{1/2}\| = \|\L^{-s}\S \Cnl^{-1/2}\|~ \|\Cnl^{1/2} q_\la(\Cn)\Cnl^{1/2}\|~\|\Cnl^{-1/2}\Cl^{1/2}\|.\]
Finally note that
\[\|\L^{-s}\S\Cnl^{-1/2}\| \leqslant \|\L^{-s}\S\Cl^{-1/2+s}\|\|\Cl^{-s}\|\|\Cl^{1/2}\Cnl^{-1/2}\|,\]
and $\|\L^{-s}\S\Cl^{-1/2+s}\| \leqslant 1$, $\|\Cl^{-s}\| \leqslant \la^{-s}$,
and moreover
\[\|\Cnl^{1/2} q_\la(\Cn)\Cnl^{1/2}\| = \sup_{\sigma \in \sigma(\Cn)} (\sigma + \la) q_\la(\sigma) \leqslant \sup_{\sigma \geqslant 0} (\sigma + \la)q_\la(\sigma) \leqslant 2c_q,\]
so, in conclusion 
\[\|\L^{-s}\S r_\la(\Cn)\| \leqslant 2c_q\la^{1/2-s} \beta, \quad \|\L^{-s}\S q_\la(\Cn) \Cl^{1/2}\| \leqslant 2 c_q \la^{-s}\beta^2.\]
The final result is obtained by gathering the upper bounds for the three terms above and the additional terms of this last section.
\end{proof}

\begin{lemma}\label{lm:bias}
Let $\la > 0$ and $s \in (0,\min(r,1/2)]$.
Under Assumption~\ref{asm:source_condition} (see Rem.~\ref{rmk:asm-source}), the following holds
\[\|\L^{-s}(\S \widehat{g}_\la - P\frho)\|_{\Ltwo} \leqslant \la^{r-s}\|\phi\|_\Ltwo.\]
\end{lemma}
\begin{proof}
Since $\S\Cl^{-1} \S^* = \L \Ll^{-1} = I - \la \Ll^{-1}$, we have
\eqals{
	\L^{-s}(\S g_\la - P\frho) &= \L^{-s}(\S\Cl^{-1}S^*\frho - P\frho) = \L^{-s}(\S\Cl^{-1}\S^*P \frho - P\frho) \\
	&= \L^{-s}(\S\Cl^{-1}\S^* - I)P\frho =  \L^{-s}(\S\Cl^{-1}\S^* - I)\L^r ~\phi \\
	&= - \la \L^{-s}\Ll^{-1} L^r \phi  = -\la^{r-s} ~ \la^{1-r+s} \Ll^{-(1-r+s)} ~ \Ll^{-(r-s)} \L^{r-s} ~ \phi,
}
from which
\eqals{
	\|\L^{-s}(\S g_\la - P\frho)\|_\Ltwo &\leqslant  \la^{r-s} \|\la^{1-r+s} \Ll^{-(1-r+s)}\| \|\Ll^{-(r-s)} \L^{r-s}\| ~ \|\phi\|_\Ltwo \\
	&\leqslant \la^{r-s} \|\phi\|_\Ltwo.
}
\end{proof}
\begin{thm}\label{thm:analytic-dec}
Let $\la > 0$ and $s \in (0,\min(r,1/2)]$.
Under Assumption~\ref{asm:source_condition} (see Rem.~\ref{rmk:asm-source}), the following holds
\[\|\L^{-s}(\S \widehat{g}_\la - P\frho)\|_{\Ltwo} \leqslant 2 \la^{-s} \beta^2 c_q \|\Cl^{-1/2}(\Sn^*\yn -  \Cn g_\la)\|_\h  + \left(1+ \beta 2 c_q \|\phi\|_\Ltwo\right) \la^{r-s} \]
where $\beta := \|\Cl^{1/2} \Cnl^{-1/2}\|$.
\end{thm}
\begin{proof}
By Prop.~\ref{prop:excess risk}, we can characterize the excess risk of $\widehat{g}_\la$ in terms of the $\Ltwo$ squared norm of $S\widehat{g}_\la - P\frho$. In this paper, simplifying the analysis of \cite{lin2018optimal}, we perform the following decomposition
\eqals{
\L^{-s}(\S \widehat{g}_\la - P\frho) &= \L^{-s}\S\widehat{g}_\la - \L^{-s}\S g_\la\\
& + ~~ \L^{-s}(\S g_\la - P\frho).
}
The first term is bounded by Lemma~\ref{lm:variance-12}, the second is bounded by Lemma~\ref{lm:bias}. 
\end{proof}
\subsection{Probabilistic bounds}
\label{sec:probbounds}
In this section denote by ${\mathcal {N}}_\infty(\la)$, the quantity
\[
{\mathcal {N}}_\infty(\la) = \sup_{x \in S} \|\Cl^{-1/2} K_x\|_\h^2,
\]
where $S \subseteq \X$ is the support of the probability measure $\rhox$.

\begin{lemma}\label{lm:L-infty}
Under Asm.~\ref{asm:hyp-space}, we have that
for any $g \in \h$
\[
\sup_{x \in \textrm{supp}(\rhox)} |g(x)| \leqslant \kappa_\mu R^{\mu} \|\C^{1/2(1-\mu)} g\|_\h = \kappa_\mu R^{\mu} \|\L^{-\mu/2} S g\|_\Ltwo. \]

\end{lemma}
\begin{proof}
	Note that, Asm.~\ref{asm:hyp-space} is equivalent to
\[
\|\C^{-1/2(1-\mu)}K_x\| \leqslant \kappa_\mu R^{\mu},\]
	for all $x$ in the support of $\rhox$.
	Then we have, for any $x$ in the support of $\rhox$,
	\eqals{
		|g(x)| &= \scal{g}{K_x}_\h = \scal{\C^{1/2(1-\mu)}g}{\C^{-1/2(1-\mu)}K_x}_\h \\
		& \leqslant \|\C^{1/2(1-\mu)}g\|_\h \|\C^{-1/2(1-\mu)}K_x\| \leqslant \kappa_\mu R^{\mu} \|\C^{1/2(1-\mu)}g\|_\h.
	}
	Now note that, since $\C^{1-\mu} = S^*\L^{-\mu}S$, we have
	\[
	\|\C^{1/2(1-\mu)}g\|_\h^2 = \scal{g}{\C^{1-\mu}g}_\h = \scal{\L^{-\mu/2}Sg}{\L^{-\mu/2}Sg}_\Ltwo.
	\]
\end{proof}
\begin{lemma}\label{lm:Ninfty-wrt-hyp-space-asm}
Under Assumption~\ref{asm:hyp-space}, we have
\[ {\mathcal {N}}_\infty(\la) \leqslant \kappa_\mu^2 R^{2\mu}\la^{-\mu}.\]
\end{lemma}
\begin{proof}
First denote with $f_{\la,u} \in \h$ the function $\Cl^{-1/2} u$ for any $u \in \h$ and $\la > 0$. Note that
\[ \|f_{\la,u}\|_\h = \|\Cl^{-1/2}u\|_\h \leqslant \|\Cl^{-1/2}\|\|u\|_\h \leqslant \la^{-1/2} \|u\|_\h.\]
Moreover, since for any $g \in \h$ the identity $\|g\|_\Ltwo = \|Sg\|_\h$, we have
\[\|f_{\la,u}\|_\Ltwo = \|\S \Cl^{-1/2}u\|_\h \leqslant \|\S \Cl^{-1/2}\|\|u\|_\h \leqslant \|u\|_\h.\]
Now denote with $B(\h)$ the unit ball in $\h$, by applying Asm.~\ref{asm:hyp-space} to $f_{\la,u}$ we have that
\eqals{
	{\mathcal {N}}_\infty(\la) & = \sup_{x \in S} \|\Cl^{-1/2}K_x\|^2 = \sup_{x \in S, u \in B(\h)} \scal{u}{\Cl^{-1/2}K_x}_\h^2 \\
	&= \sup_{x\in S, u \in B(\h)} \scal{f_{\la,u}}{K_x}_\h^2  = \sup_{u \in B(\h)} \sup_{x \in S} |f_{\la,u}(x)|^2\\
	& \leqslant \kappa_\mu^2 R^{2\mu}\sup_{u \in B(\h)} \|f_{\la,u}\|^{2\mu}_\h \|f_{\la,u}\|^{2-2\mu}_\Ltwo  \\
	& \leqslant \kappa_\mu^2 R^{2\mu}\la^{-\mu} \sup_{u \in B(\h)} \|u\|_\h^2 \leqslant \kappa_\mu^2 R^{2\mu}\la^{-\mu}.
}
\end{proof}
\begin{lemma}\label{lem:bounding-beta}
Let $\la > 0$, $\delta \in (0,1]$ and $n \in \N$. Under Assumption~\ref{asm:hyp-space}, we have that,
when 
\[ n \geqslant 11(1+\kappa_\mu^2 R^{2\mu} \la^{-\mu}) \log \frac{8R^2}{\la\delta},\]
 then the following holds with probability $1-\delta$,
\[\|\Cl^{1/2} \Cnl^{-1/2}\|^2 \leqslant 2.\]
\end{lemma}
\begin{proof}
This result is a refinement of the one in \cite{rudi2013sample} and is based on non-commutative Bernstein inequalities for random matrices \cite{tropp2012user}. By Prop.~8 in \cite{rudi2017generalization}, we have that
\[\|\Cl^{1/2} \Cnl^{-1/2}\|^2 \leqslant (1-t)^{-1}, \quad t := \|\Cl^{-1/2}(\C - \Cn)\Cl^{-1/2}\|.\]
When $0 < \la \leqslant \|\C\|$, by Prop.~6 of \cite{rudi2017generalization} (see also \cite{rudi2017falkon} Lemma~9 for more refined constants), we have
that the following holds with probability at least $1-\delta$,
\[t \leqslant \frac{2\eta(1+{\mathcal {N}}_\infty(\la))}{3 n} + \sqrt{\frac{2\eta {\mathcal {N}}_\infty(\la)}{n}},\]
with $\eta = \log \frac{8R^2}{\la\delta}$.
Finally, by selecting $n \geqslant 11(1+ \kappa_\mu^2 R^{2\mu}\la^{-\mu}) \eta$, we have that $t \leqslant 1/2$ and so $\|\Cl^{1/2} \Cnl^{-1/2}\|^2 \leqslant (1-t)^{-1} \leqslant 2$, with probability $1-\delta$. 

To conclude note that when $\la \geqslant \|\C\|$, we have
\[\|\Cl^{1/2} \Cnl^{-1/2}\|^2 \leqslant \|\C + \la I\|\|(\Cn + \la I)^{-1}\| \leqslant \frac{\|\C\| + \la}{\la} = 1 + \frac{\|\C\|}{\la} \leqslant 2.\]
\end{proof}
\begin{lemma}\label{lm:bounding-variance-first}
Under Assumption~\ref{asm:hyp-space},~\ref{asm:capacity_condition},~\ref{asm:source_condition} (see Rem.~\ref{rmk:asm-source}),~\ref{asm:reg-Pfrho} we have
\begin{enumerate}
\item Let $\la > 0$, $n \in \N$, the following holds
\[\expect{\|\Cl^{-1/2}(\Sn^*\yn -  \Cn g_\la)\|_\h^2} \leqslant \|\phi\|^2_\Ltwo\la^{2r} + \frac{2 \kappa_\mu^2 R^{2\mu} \la^{-(\mu - 2r)}}{n} + \frac{4\kappa_\mu^2 R^{2\mu}AQ\la^{-\frac{q + \mu\alpha }{q\alpha+\alpha}}}{n},\]
where $A := \|\frho - P\frho\|_{L^{2q}(\X,\rhox)}^{2 - 2/(q+1)}$.

\item Let $\delta \in (0,1]$, under the same assumptions, the following holds with probability at least $1-\delta$
\eqals{
\|\Cl^{-1/2}(\Sn^*\yn -  \Cn g_\la)\|_\h &\leqslant c_0\la^{r} + \frac{4(c_1\la^{-\frac{\mu}{2}}  + c_2\la^{-r-\mu})\log\frac{2}{\delta}}{n} \\
&\qquad\qquad + \sqrt{\frac{16 \kappa_\mu^2 R^{2\mu} (\la^{-(\mu - 2r)}  +  2AQ\la^{-\frac{q + \mu\alpha }{q\alpha+\alpha}})\log\frac{2}{\delta}}{n}},
}
with $c_0 = \|\phi\|_\Ltwo$, $c_1 = \kappa_\mu R^{\mu} M  + \kappa_\mu^2 R^{2\mu}(2R)^{2r-\mu} \|\phi\|_\Ltwo$, $c_2 = \kappa_\mu^2 R^{2\mu} \|\phi\|_\Ltwo$
\end{enumerate}
\end{lemma}
\begin{proof}
First denote with $\zeta_i$ the random variable
\[\zeta_i = (y_i  -  g_\la(x_i)) \Cl^{-1/2}K_{x_i}.\]
In particular note that, by using the definitions of $\Sn$, $\yn$ and $\Cn$, we have
\[\Cl^{-1/2}(\Sn^*\yn -  \Cn g_\la) = \Cl^{-1/2}(\frac{1}{n} \sum_{i=1}^n K_{x_i} y_i -  \frac{1}{n} (K_{x_i} \otimes K_{x_i})g_\la) =  \frac{1}{n} \sum_{i=1}^n \zeta_i.\]
I
So, by noting that $\zeta_i$ are independent and identically distributed, we have
\eqals{
\expect{\|\Cl^{-1/2}(\Sn^*\yn -  \Cn g_\la)\|^2_\h} &= \expect{\|\frac{1}{n} \sum_{i=1}^n \zeta_i\|_\h^2} = \frac{1}{n^2}\sum_{i,j=1}^n\expect{\scal{\zeta_i}{\zeta_j}_\h} \\
&= \frac{1}{n} \expect{\|\zeta_1\|^2_\h } + \frac{n-1}{n} \|\expect{\zeta_1}\|^2_\h.
}
Now note that
\[\expect{\zeta_1} = \Cl^{-1/2}(\expect{K_{x_1} y_1} - \expect{K_{x_1} \otimes K_{x_1}} g_\la) = \Cl^{-1/2}(\S^*\frho - \C g_\la).\]
In particular, by the fact that $\S^*\frho = P\frho$, $P\frho = \L^r \phi$ and $\C g_\la = \C \Cl^{-1} \S^*\frho$ and $\C \Cl^{-1} = I - \la \Cl^{-1}$, we have
\[\Cl^{-1/2}(\S^*\frho - \C g_\la) = \la \Cl^{-3/2} \S^*\frho = \la^r ~\la^{1-r}\Cl^{-(1-r)} ~\Cl^{-1/2-r} \S^* \L^r ~ \phi.\]
So, since $\|\Cl^{-1/2-r} \S^*\L^r\| \leqslant 1$, as proven in Eq.~\ref{eq:ClaSL}, then
\[\|\expect{\zeta_1}\|_\h \leqslant \la^r \|\la^{1-r}\Cl^{-(1-r)}\| ~\|\Cl^{-1/2-r} \S^* \L^r\| ~ \|\phi\|_\Ltwo \leqslant \la^r \|\phi\|_\Ltwo := Z.\]
Morever
\eqals{
	\expect{\|\zeta_1\|_\h^2} &=  \expect{\|\Cl^{-1/2}K_{x_1}\|_\h^2 (y_1 - g_\la(x_1))^2} = \mathbb{E}_{x_1} \mathbb{E}_{y_1|x_1}[\|\Cl^{-1/2}K_{x_1}\|_\h^2 (y_1 - g_\la(x_1))^2] \\
	&= \mathbb{E}_{x_1} [\|\Cl^{-1/2}K_{x_1}\|_\h^2 (\frho(x_1) - g_\la(x_1))^2].
}
Moreover we have
\eqals{
\expect{\|\zeta_1\|_\h^2} &= \mathbb{E}_{x} [\|\Cl^{-1/2}K_{x}\|_\h^2 (\frho(x) - g_\la(x))^2] \\
& =\mathbb{E}_{x} [\|\Cl^{-1/2}K_{x}\|_\h^2 ((\frho(x) - (P\frho)(x)) + ((P\frho)(x) - g_\la(x)))^2] \\
& \leqslant 2\mathbb{E}_{x} [\|\Cl^{-1/2}K_{x}\|_\h^2 (\frho(x) - (P\frho)(x))^2] + 2\mathbb{E}_{x} [\|\Cl^{-1/2}K_{x}\|_\h^2 ((P\frho)(x) - g_\la(x))^2].
}
Now since $\expect{A B} \leqslant (\textrm{ess}\sup A) \expect{B}$, for any two random variables $A, B$, we have
\eqals{
\mathbb{E}_{x} [\|\Cl^{-1/2}K_{x}\|_\h^2 ((P\frho)(x) - g_\la(x))^2] &\leqslant {\mathcal {N}}_\infty(\la) \mathbb{E}_{x} [((P\frho)(x) - g_\la(x))^2]  \\
& = {\mathcal {N}}_\infty(\la) \|P\frho - Sg_\la\|^2_\Ltwo\\ & \leqslant \kappa_\mu^2 R^{2\mu} \la^{-(\mu-2r)},
}
where in the last step we bounded ${\mathcal {N}}_\infty(\la)$ via Lemma~\ref{lm:Ninfty-wrt-hyp-space-asm} and $\|P\frho - \S g_\la\|^2_\Ltwo$, via Lemma.~\ref{lm:bias} applied with $s = 0$.
Finally, denoting by $a(x) = \|\Cl^{-1/2}K_{x}\|_\h^2$ and $b(x) = (\frho(x) - (P\frho)(x))^2$ and noting that by Markov inequality we have
$\mathbb{E}_{x}[{\bf 1}_{\{b(x) > t\}}] = \rhox(\{b(x) > t\}) = \rhox(\{b(x)^q > t^q\}) \leqslant \mathbb{E}_{x}[b(x)^q] t^{-q}$, for any $t > 0$. Then for any $t > 0$ the following holds
\eqals{
\mathbb{E}_{x}[a(x)b(x)] &= \mathbb{E}_{x}[a(x)b(x) {\bf 1}_{\{b(x) \leqslant t\}}] + \mathbb{E}_{x}[a(x)b(x) {\bf 1}_{\{b(x) > t\}}] \\
& \leqslant  t \mathbb{E}_{x}[a(x)] + N_\infty(\la)\mathbb{E}_{x}[b(x) {\bf 1}_{\{b(x) > t\}}] \\ 
& \leqslant  t N(\la) + N_\infty(\la) \mathbb{E}_{x}[b(x)^q] t^{-q}.
}
By minimizing the quantity above in $t$, we obtain
\eqals{
\mathbb{E}_{x} [\|\Cl^{-1/2}K_{x}\|_\h^2 (\frho(x) - (P\frho)(x))^2] &\leqslant 2\|\frho - P\frho\|_{L^q(X,\rhox)}^{\frac{q}{q+1}} {\mathcal {N}}(\la)^{\frac{q}{q+1}} {\mathcal {N}}_\infty(\la)^{\frac{1}{q+1}}\\
& \leqslant 2\kappa_\mu^2 R^{2\mu}AQ\la^{-\frac{q + \mu\alpha }{q\alpha+\alpha}}.
}
So finally
\[\expect{\|\zeta_1\|^2_\h} \leqslant 2\kappa_\mu^2 R^{2\mu} \la^{-(\mu-2r)} + 4\kappa_\mu^2 R^{2\mu}AQ\la^{-\frac{q + \mu\alpha }{q\alpha+\alpha}} := W^2.\]

To conclude the proof, let us obtain the bound in high probability. We need to bound the higher moments of $\zeta_1$. First note that 
\[\expect{\|\zeta_1 - \expect{\zeta_1}\|_\h^p} \leqslant \expect{\|\zeta_1 - \zeta_2\|_\h^p} \leqslant 2^{p-1}\expect{\|\zeta_1|^p_\h + \|\zeta_2\|_\h^p} \leqslant 2^{p}\expect{\|\zeta_1|^p_\h}.\]
Moreover, denoting by $S \subseteq \X$ the support of $\rhox$ and recalling that $y$ is bounded in $[-M, M]$, the following bound holds almost surely
\eqals{
\|\zeta_1\| &\leqslant \sup_{x \in S}\|\Cl^{-1/2}K_{x}\| (M + |g_\la(x)|) \leqslant (\sup_{x \in S}\|\Cl^{-1/2}K_{x}\|) (M + \sup_{x \in S} |g_\la(x)|) \\
& \leqslant \kappa_\mu R^{\mu} \la^{-\mu/2} (M + \kappa_\mu R^{\mu} \|\C^{1/2(1-\mu)} g_\la\|_\h). 
}
where in the last step we applied Lemma~\ref{lm:Ninfty-wrt-hyp-space-asm} and Lemma~\ref{lm:L-infty}. In particular, by definition of $g_\la$, the fact that $\S^* \frho = \S^* P \frho$, that $P\frho = \L^r \phi$ and that $\|\Cl^{-(1/2+r)} \S^* \L^r\| \leqslant 1$ as proven in Eq.~\ref{eq:ClaSL}, we have
\eqals{
\|\C^{1/2(1-\mu)}g_\la\|_\h & = \|\C^{1/2(1-\mu)} \Cl^{-1}S^* \L^r \phi\|_\h \\
& \leqslant \|\C^{1/2(1-\mu)} \C^{-1/2(1-\mu)}\|\|\Cl^{-(\mu/2 - r)}\|\|\Cl^{-(1/2+r)} \S^* \L^r\| \|\phi\|_\Ltwo \\
	&\leqslant \|\Cl^{r - \mu/2}\|\|\phi\|_\Ltwo.
}
Finally note that if $r \leqslant \mu/2$ then $\|\Cl^{r - \mu/2}\| \leqslant \la^{-(\mu/2-r)}$, if
$r \geqslant \mu/2$ then 
\[\|\Cl^{r - \mu/2}\| = (\|C\| + \la)^{r-\mu/2} \leqslant (2\|C\|)^{r-\mu/2} \leqslant (2R)^{2r-\mu}.\]
So in particular
\[\|\Cl^{r - \mu/2}\| \leqslant (2R)^{2r-\mu} + \la^{-(\mu/2-r)}.\]
Then the following holds almost surely
\[\|\zeta_1\| \leqslant (\kappa_\mu R^{\mu} M  + \kappa_\mu^2 R^{2\mu}(2R)^{2r-\mu} \|\phi\|_\Ltwo) \la^{-\mu/2} + \kappa_\mu^2 R^{2\mu} \|\phi\|_\Ltwo \la^{r-\mu} := V.\]
So finally
\[\expect{\|\zeta_1 - \expect{\zeta_1}\|_\h^p} \leqslant 2^p \expect{\|\zeta_1\|_\h^p} \leqslant \frac{p!}{2} (2V)^{p-2} (4W^2).\]
By applying Pinelis inequality, the following holds with probability $1-\delta$
\[\|\frac{1}{n} \sum_{i=1}^n (\zeta_i - \expect{\zeta_i}) \|_\h \leqslant \frac{4V \log\frac{2}{\delta}}{n} + \sqrt{\frac{8W \log\frac{2}{\delta}}{n}}.\]
So with the same probability
\[\|\frac{1}{n} \sum_{i=1}^n \zeta_i \|_\h \leqslant \|\frac{1}{n} \sum_{i=1}^n (\zeta_i - \expect{\zeta_i}) \|_\h + \|\expect{\zeta_1}\|_\h \leqslant Z + \frac{4V \log\frac{2}{\delta}}{n} + \sqrt{\frac{8W \log\frac{2}{\delta}}{n}}.\]
\end{proof}

\begin{lemma}\label{lm:ctrl-ghatla-prob}
Let $\la >0$, $n \in \N$ and $s \in (0, 1/2]$. Let $\delta \in (0,1]$.
Under Assumption~\ref{asm:hyp-space},~\ref{asm:capacity_condition},~\ref{asm:source_condition} (see Rem.~\ref{rmk:asm-source}),~\ref{asm:reg-Pfrho}, when 
\[n \geqslant 11(1+\kappa_\mu^2 R^{2\mu} \la^{-\mu}) \log \frac{16R^2}{\la \delta},\]
then the following holds with probability $1-\delta$,
\eqals{
	\|\L^{-s}\S(\widehat{g}_\la - g_\la)\|_{\Ltwo} &~~\leq~~ c_0\la^{r-s} ~+~ \frac{(c_1 \la^{-\frac{\mu}{2}-s}  + c_2\la^{r-\mu-s})\log\frac{4}{\delta}}{n} \\
	&\qquad\qquad + \sqrt{\frac{ (c_3\la^{-(\mu +2s - 2r)}  +  c_4\la^{-\frac{q + \mu\alpha }{q\alpha+\alpha}-2s})\log\frac{4}{\delta}}{n}}.
}
with $c_0 = 7c_q \|\phi\|_\Ltwo$, $c_1 = 16c_q(\kappa_\mu R^{\mu} M  + \kappa_\mu^2 R^{2\mu}(2R)^{2r-\mu} \|\phi\|_\Ltwo)$,  $c_2 = 16c_q  \kappa_\mu^2 R^{2\mu} \|\phi\|_\Ltwo$, $c_3 = 64\kappa_\mu^2 R^{2\mu}c_q^2$, $c_4 = 128\kappa_\mu^2 R^{2\mu}AQc_q^2$.
\end{lemma}
\begin{proof}
Let $\tau = \delta/2$, the result is obtained by combining Lemma~\ref{lm:variance-12}, with Lemma~\ref{lm:bounding-variance-first} with probability $\tau$, and Lemma~\ref{lem:bounding-beta}, with probability $\tau$ and then taking the intersection bound of the two events.
\end{proof}

\begin{cor}\label{cor:bound-gla-Ltwo}
Let $\la > 0$, $n \in \N$ and $s \in (0, 1/2]$. Let $\delta \in (0,1]$.
Under the assumptions of Lemma~\ref{lm:ctrl-ghatla-prob}, when 
\[n \geqslant 11(1+\kappa_\mu^2 R^{2\mu} \la^{-\mu}) \log \frac{16R^2}{\la \delta},\]
then the following holds with probability $1-\delta$,
\eqals{
\|\L^{-s} \S \widehat{g}_\la\|_\Ltwo & ~~\leq~~ R^{2r-2s} + (1+c_0)\la^{r-s}~+~ \frac{(c_1 \la^{-\frac{\mu}{2}-s}  + c_2\la^{r-\mu-s})\log\frac{4}{\delta}}{n} \\
&\qquad\qquad + \sqrt{\frac{ (c_3\la^{-(\mu +2s - 2r)}  +  c_4\la^{-\frac{q + \mu\alpha }{q\alpha+\alpha}-2s})\log\frac{4}{\delta}}{n}}  + 
}
with the same constants $c_0, \dots, c_4$ as in Lemma~\ref{lm:ctrl-ghatla-prob}.
\end{cor}
\begin{proof}
First note that 
\[\|\L^{-s} \S \widehat{g}_\la\|_\Ltwo \leqslant \|\L^{-s} \S (\widehat{g}_\la - g_\la)\|_\Ltwo + \|\L^{-s} \S g_\la\|_\Ltwo.\]
The first term on the right hand side is controlled by Lemma~\ref{lm:ctrl-ghatla-prob}, for the second, by using the definition of $g_\la$ and Asm.~\ref{asm:source_condition} (see Rem.~\ref{rmk:asm-source}), we have
\eqals{
\|\L^{-s} \S g_\la\|_\Ltwo &\leqslant \|\L^{-s}\S \Cl^{-1/2+s}\|\|\Cl^{-(s - r)}\|\|\Cl^{-1/2-r}\S^*\L^r\|\|\phi\|_\Ltwo \\
& \leqslant \|\Cl^{r-s}\|\|\phi\|_\Ltwo,
}
where $\|\Cl^{-1/2-r}\S^*\L^r\|  \leqslant 1$ by Eq.~\ref{eq:ClaSL} and analogously $\|\L^{-s}\S \Cl^{-1/2+s}\| \leqslant 1$.
Note that if $s \geqslant r$ then $\|\Cl^{r-s}\| \leqslant \la^{-(s-r)}$. If $s < r$, we have
\[\|\Cl^{r-s}\| = (\|\C\| + \la)^{r-s} \leqslant \|C\|^{r-s} + \la^{r-s} \leqslant R^{2r-2s} + \la^{r-s}.\]
So finally $\|\Cl^{r-s}\| \leqslant R^{2r-2s} + \la^{r-s}$.
\end{proof}

\begin{cor}\label{cor:bound-gla-Linfty}
Let $\la > 0$, $n \in \N$ and $s \in (0, 1/2]$. Let $\delta \in (0,1]$.
Under Assumption~\ref{asm:hyp-space},~\ref{asm:capacity_condition},~\ref{asm:source_condition} (see Rem.~\ref{rmk:asm-source}),~\ref{asm:reg-Pfrho}, when 
\begin{align*}n \geqslant 11(1+\kappa_\mu^2 R^{2\mu} \la^{-\mu}) \log \frac{16R^2}{\la \delta},
\end{align*}
 then the following holds with probability $1-\delta$,
\begin{align*}
	\sup_{x \in \X} |\widehat{g}_\la(x)| & ~~\leq~~ \kappa_\mu R^{\mu}R^{2r-2s}~+~\kappa_\mu R^{\mu}(1+c_0)\la^{r-\mu/2} ~+~ \kappa_\mu R^{\mu} \frac{(c_1 \la^{-\mu}  + c_2\la^{r - 3/2\mu})\log\frac{4}{\delta}}{n} \\
	&\qquad\qquad + \kappa_\mu R^{\mu}\sqrt{\frac{ (c_3\la^{-(2\mu - 2r)}  +  \kappa_\mu R^{\mu} c_4\la^{-\frac{q + \mu\alpha }{q\alpha+\alpha}-\mu})\log\frac{4}{\delta}}{n}}.
\end{align*}
with the same constants $c_0,\dots, c_4$ in Lemma~\ref{lm:ctrl-ghatla-prob}.
\end{cor}
\begin{proof}
The proof is obtained by applying Lemma~\ref{lm:L-infty} on $\widehat{g}_\la$ and then Corollary~\ref{cor:bound-gla-Ltwo}.
\end{proof}

\subsection{Main Result}
\label{app:main_result}
\begin{thm}\label{thm:final-filters}
Let $\la > 0$, $n \in \N$ and $s \in (0, \min(r,1/2)]$.
Under Assumption~\ref{asm:hyp-space},~\ref{asm:capacity_condition},~\ref{asm:source_condition} (see Rem.~\ref{rmk:asm-source}),~\ref{asm:reg-Pfrho}, when 
\[n \geqslant 11(1+\kappa_\mu^2 R^{2\mu} \la^{-\mu}) \log \frac{c_0}{\la^{3+4r-4s}},\]
then
\[\expect{\|\L^{-s}(\S\widehat{g}_\la - P\frho)\|^2_\Ltwo} \leqslant c_1\frac{\la^{-(\mu +2s- 2r)}}{n} + c_2\frac{\la^{-\frac{q + \mu\alpha }{q\alpha+\alpha}-2s}}{n} +  c_3\la^{2r-2s},\]
where $m_4 = M^4$, $c_0 = 32R^{4-4s} m_4 + 32R^{8-8r-8s}\|\phi\|^4_\Ltwo$, $c_1 = 16c_q^2 \kappa_\mu^2 R^{2\mu}$, $c_2 = 32 c_q^2\kappa_\mu^2 R^{2\mu}AQ$, $c_3 = 3+ 8 c_q^2 \|\phi\|_\Ltwo^2$.
\end{thm}
\begin{proof}
Denote by $R(\widehat{g}_\la)$, the expected risk $R(\widehat{g}_\la) = \EE(\widehat{g}_\la) - \inf_{g \in \h} \EE(g)$.
First, note that by Prop.~\ref{prop:excess risk}, we have 
\[R_s(\widehat{g}_\la) = \|\L^{-s}(\S\widehat{g}_\la - P\frho)\|^2_\Ltwo.\]
Denote by $E$ the event such that $\beta$ as defined in Thm.~\ref{thm:analytic-dec}, satisfies $\beta \leqslant 2$.
Then we have
\eqals{
\expect{R_s(\widehat{g}_\la)} &= \expect{R_s(\widehat{g}_\la) {\bf 1}_E} + \expect{R(\widehat{g}_\la) {\bf 1}_{E^c}}.
}
For the first term, by Thm.~\ref{thm:analytic-dec} and Lemma~\ref{lm:bounding-variance-first}, we have
\eqals{
\mathbb{E}[R_s(\widehat{g}_\la) & {\bf 1}_E] \leqslant  \expect{\Big(2\la^{-2s}\beta^4 c_q^2 \|\Cl^{-1/2}(\Sn^*\yn -  \Cn g_\la)\|_\h^2  \\
	& \qquad\qquad\qquad\qquad\qquad + 2\left(1+ \beta^2 2 c_q^2 \|\phi\|_\Ltwo^2\right) \la^{2r-2s}\Big) {\bf 1}_E} \\
& \leqslant 8 \la^{-2s} c_q^2 \expect{\|\Cl^{-1/2}(\Sn^*\yn -  \Cn g_\la)\|_\h^2}  + 2\left(1+ 4 c_q^2 \|\phi\|_\Ltwo^2\right) \la^{2r-2s} \\
& \leqslant \frac{16c_q^2 \kappa_\mu^2 R^{2\mu} \la^{-\mu + 2r-2s}}{n} + \frac{32 c_q^2\kappa_\mu^2 R^{2\mu}AQ\la^{-\frac{q + \mu\alpha }{q\alpha+\alpha} - 2s}}{n} + \left(2+ 8 c_q^2 \|\phi\|_\Ltwo^2\right) \la^{2r-2s}.
}
For the second term, since $\Cnl^{1/2}q_\la(\Cn)\Cnl^{1/2} = \Cnl q_\la(\Cn) \leqslant \sup_{\sigma > 0} (\sigma + \la) q_\la(\sigma) \leqslant c_q$ by definition of filters, and that $P\frho = L^r \phi$, we have
\eqals{
	R_s(\widehat{g}_\la)^{1/2} &\leqslant \|\L^{-s}\S\widehat{g}_\la\|_\Ltwo + \|\L^{-s} P\frho\|_\Ltwo \\
	&\leqslant \|\L^{-s}\S\|\|\Cnl^{-1/2}\| \|\Cnl^{1/2}q_\la(\Cn)\Cnl^{1/2}\| \|\Cnl^{-1/2} \Sn^*\|\|\yn\| + \|\L^{-s} \L^{r}\| \|\phi\|_\Ltwo\\
	& \leqslant R^{1/2-s}\la^{-1/2}\|\yn\| + R^{2r-2s}\|\phi\|_\Ltwo \\
	& \leqslant \la^{-1/2}(R^{1/2-s}(n^{-1}\sum_{i=1}^n y_i) + R^{1+2r-2s}\|\phi\|_\Ltwo),
}
where the last step is due to the fact that $1 \leqslant \la^{-1/2}\|\L\|^{1/2}$ since $\la$ satisfies $0 < \la \leqslant \|\C\| =\|\L\| \leqslant R^2$.
Denote with $\delta$ the quantity $\delta = \la^{2 + 4r-4s}/c_0$.
Since $\expect{{\bf 1}_E^c}$ corresponds to the probability of the event $E^c$, and, by Lemma~\ref{lem:bounding-beta}, we have that $E^c$ holds with probability at most $\delta$ since $n \geqslant 11(1+\kappa_\mu^2 R^{2\mu} \la^{-\mu}) \log \frac{8R^2}{\la\delta}$, then 
we have that
\eqals{
\expect{R(\widehat{g}_\la) {\bf 1}_{E^c}} &\leqslant \expect{\|S \widehat{g}_\la\|^2_\Ltwo {\bf 1}_{E^c}} \leqslant \sqrt{\expect{\|S \widehat{g}_\la\|^4_\Ltwo}} \sqrt{\expect{{\bf 1}_{E^c}}}\\
& \leqslant \sqrt{\frac{4R^{2-4s} n^{-2}(\sum_{i,j=1}^n \expect{y_i^2 y_j^2}) + 4R^{4-8r-8s}\|\phi\|^4_\Ltwo}{\la^2}} \sqrt{\delta} \\
& \leqslant \frac{\sqrt{\delta}}{\la} \sqrt{4R^{2-4s} m_4 + 4R^{4-8r-8s}\|\phi\|^4_\Ltwo} \\
&= \frac{\sqrt{\delta c_0/(8R^2)}}{\la} \leqslant \la^{2r-2s}.
}
\end{proof}

\begin{cor}
\label{cor:final_bound_lambda}
Let $\la > 0$ and $n \in \N$ and $s = 0$. Under Assumption~\ref{asm:hyp-space},~\ref{asm:capacity_condition},~\ref{asm:source_condition} (see Rem.~\ref{rmk:asm-source}),~\ref{asm:reg-Pfrho}, when
\eqal{\label{eq:la-wrt-n}
	\la ~~=~~ B_1~ \begin{cases}
		n^{-\alpha/\left(2r\alpha + 1 + \frac{\mu\alpha-1}{q+1}\right)} & 2r\alpha + 1 + \frac{\mu\alpha-1}{q+1} > \mu\alpha \\
		n^{-1/\mu}~(\log B_2 n)^{\frac{1}{\mu}} & 2r\alpha + 1 + \frac{\mu\alpha-1}{q+1} \leqslant \mu\alpha.
	\end{cases}
}
then,
\eqal{
\mathbb{E}~\EE(\widehat{g}_\la) - \inf_{g \in \h} \EE(g) \leqslant B_3 \begin{cases}
n^{-2r\alpha/\left(2r\alpha + 1 + \frac{\mu\alpha-1}{q+1}\right)} & 2r\alpha + 1 + \frac{\mu\alpha-1}{q+1} > \mu\alpha \\
n^{-2r/\mu} & 2r\alpha + 1 + \frac{\mu\alpha-1}{q+1} \leqslant \mu\alpha
\end{cases}
}
where $B_2 = 3 \vee (32R^6 m_4)^{\frac{\mu}{3+4r}}B_1^{-\mu}$ and $B_1$ defined explicitly in the proof.
\end{cor}
\begin{proof}
The proof of this corollary is a direct application of Thm.~\ref{thm:final-filters}. In the rest of the proof we find the constants to guarantee that the condition relating $n, \la$ in the theorem is always satisfied. Indeed 
to guarantee the applicability of Thm.~\ref{thm:final-filters}, we need to be sure that $n \geqslant 11(1+\kappa_\mu^2 R^{2\mu} \la^{-\mu}) \log \frac{32R^6 m_4}{\la^{3+4r}}$. This is satisfied when both the following conditions hold $n \geqslant 22 \log \frac{32R^6 m_4}{\la^{3+4r}}$ and $n \geqslant 2\kappa_\mu^2 R^{2\mu} \la^{-\mu} \log \frac{32R^6 m_4}{\la^{3+4r}}$. 
To study the last two conditions, we recall that for $A, B,s, q > 0$ we have that $A n^{-s} \log (B n^q)$ satisfy
\[ A n^{-s} \log(B n^{q}) = \frac{qA B^{s/q}}{s} \frac{\log B^{s/q} n^s}{B^{s/q} n^s} \leqslant  \frac{qA B^{s/q}}{e s},\]
for any $n > 0$, since $\frac{log x}{x} \leqslant \frac{1}{e}$ for any $x > 0$.
Now we define explicitly $B_1$, let $\tau = \alpha/\left(2r\alpha + 1 + \frac{\mu\alpha-1}{q+1}\right)$, we have
\eqal{B_1 &= \left(\frac{22(3+4r)}{e\mu}(32R^6 m_4)^{\frac{\mu}{3+4r}}\right)^\frac{1}{\mu}  \vee \\
& \qquad\qquad	 \vee \begin{cases}
\left(\frac{2M(3+4r)}{e(1/\tau - \mu)}(32R^6 m_4)^{\frac{1/\tau-\mu}{3+4r}}\right)^\tau &  2r\alpha + 1 + \frac{\mu\alpha-1}{q+1} > \mu\alpha \\
\left(\frac{2M(3+4r)}{\mu}\right)^\frac{1}{\mu} & 2r\alpha + 1 + \frac{\mu\alpha-1}{q+1} \leqslant \mu\alpha
\end{cases}.
}
For the first condition, we use the fact that $\la$ is always larger than $B_1 n^{-1/\mu}$, so we have
\[\frac{22}{n} \log \frac{32R^6 m_4}{\la^{3+4r}} \leqslant \frac{22}{n} \log \frac{32R^6 m_4 n^{(3+4r)/\mu}}{B_1^{3+4r}} \leqslant \frac{22(3+4r) (32R^6 m_4)^{\mu/(3+4r)}}{e \mu B_1^\mu} \leqslant 1.\]
For the second inequality, when $2r\alpha + 1 + \frac{\mu\alpha-1}{q+1} \geqslant \mu\alpha$, we have $\la = B_1 n^{-\tau}$, so
\eqals{
\frac{2\kappa_\mu^2 R^{2\mu}}{n} \la^{-\mu} \log \frac{32R^6 m_4}{\la^{3+4r}} &\leqslant \frac{2\kappa_\mu^2 R^{2\mu}}{B_1^{\mu} n^{1 - \mu\tau}} \log \frac{32R^6 m_4 n^{(3+4r)\tau}}{B_1^{3+4r}} \\
&\leqslant \frac{2\kappa_\mu^2 R^{2\mu}(3+4r)\tau}{e(1-\mu\tau)}\frac{(32R^6 m_4)^{\frac{1/\tau-\mu}{3+4r}}}{B_1^{1/\tau}} \leqslant 1.
}
Finally, when $2r\alpha + 1 + \frac{\mu\alpha-1}{q+1} \geqslant \mu\alpha$, we have $\la = B_1 n^{-1/\mu} (\log B_2 n)^{1/\mu}$. So since $\log(B_2 n) > 1$, we have
\eqals{
\frac{2\kappa_\mu^2 R^{2\mu}}{n} \log \frac{32R^6 m_4}{\la^{3+4r}} & \leqslant \frac{2\kappa_\mu^2 R^{2\mu}}{B_1^{\mu}} \frac{\log \frac{32R^6 m_4 n^{(3+4r)/\mu}}{B_1^{3+4r}}}{\log(B_2 n)} = \frac{2\kappa_\mu^2 R^{2\mu}(3+4r)}{\mu B_1^{\mu}} \frac{\log \frac{(32R^6 m_4)^{\mu/(3+4r)} n}{B_1^\mu}}{\log(B_2 n)} \leqslant 1.
}
So by selecting $\la$ as in Eq.~\ref{eq:la-wrt-n}, we guarantee that the condition required by Thm.~\ref{thm:final-filters} is satisfied. 

Finally the constant $B_3$ is obtained by 
\[B_3 = c_1 \max(1,w)^{-(\mu +2s- 2r)} + c_2\max(1,w)^{-\frac{q + \mu\alpha }{q\alpha+\alpha}-2s}+  c_3\max(1,w)^{2r-2s},\]
with $w = B_1 \log (1 + B_2)$ and $c_1, c_2, c_3$ as in Thm.~\ref{thm:final-filters}.
\end{proof}

\section{Experiments with different sampling}

\label{app:experiments}

We present here the results for two different types of sampling, which seem to be more stable, perform better and are widely used in practice : \\
{\bfseries Without replacement (Figure \ref{fig:t_versus_n_cycling})}: for which we select randomly the data points but never use two times over the same point in one epoch.\\
{\bfseries Cycles (Figure \ref{fig:t_versus_n_without_replacement})}: for which we pick successively the data points in the same order. 

\begin{figure}[ht]
\footnotesize
\includegraphics[width=0.48\textwidth]{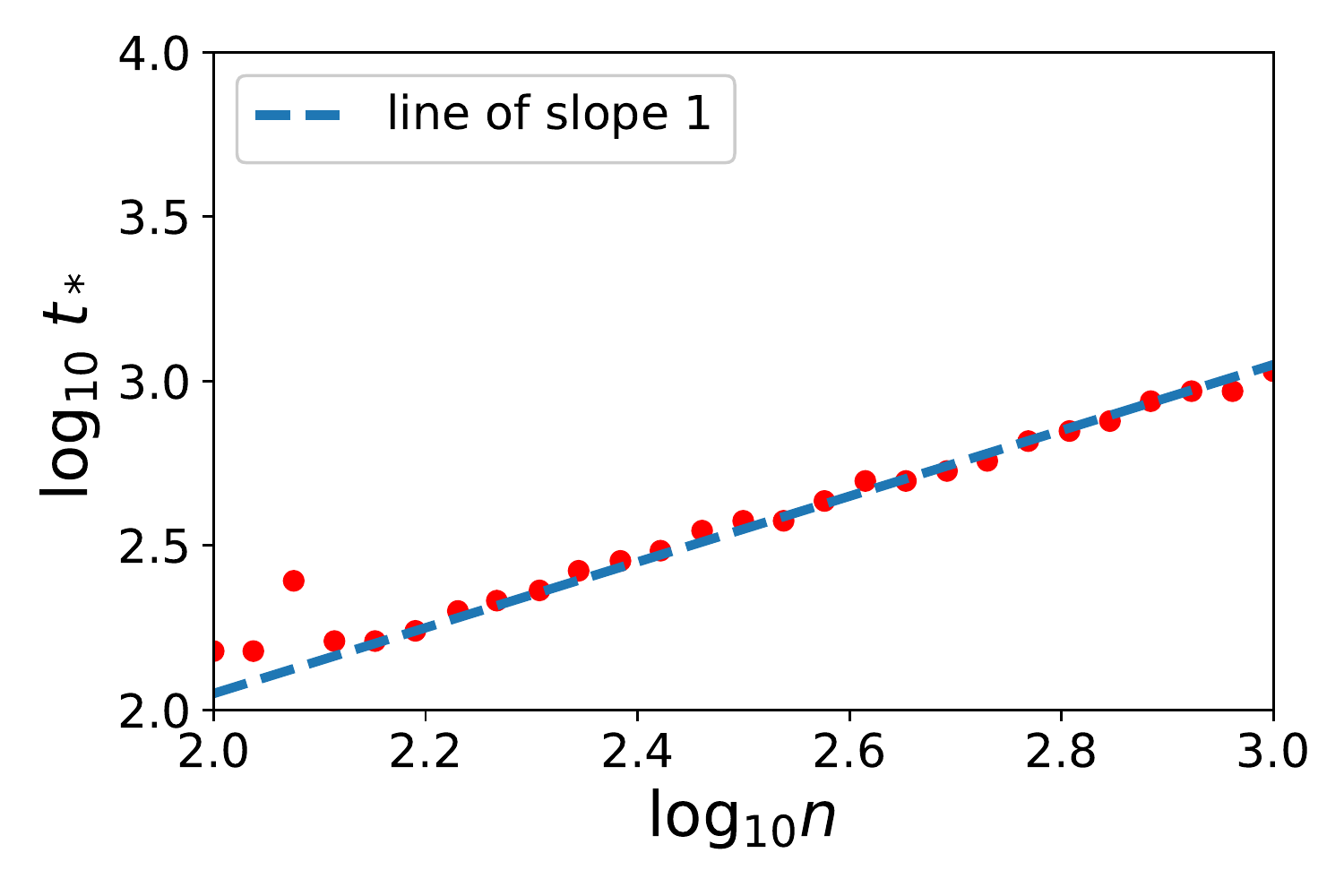}
\hspace{0.5cm}%
\includegraphics[width=0.48\textwidth]{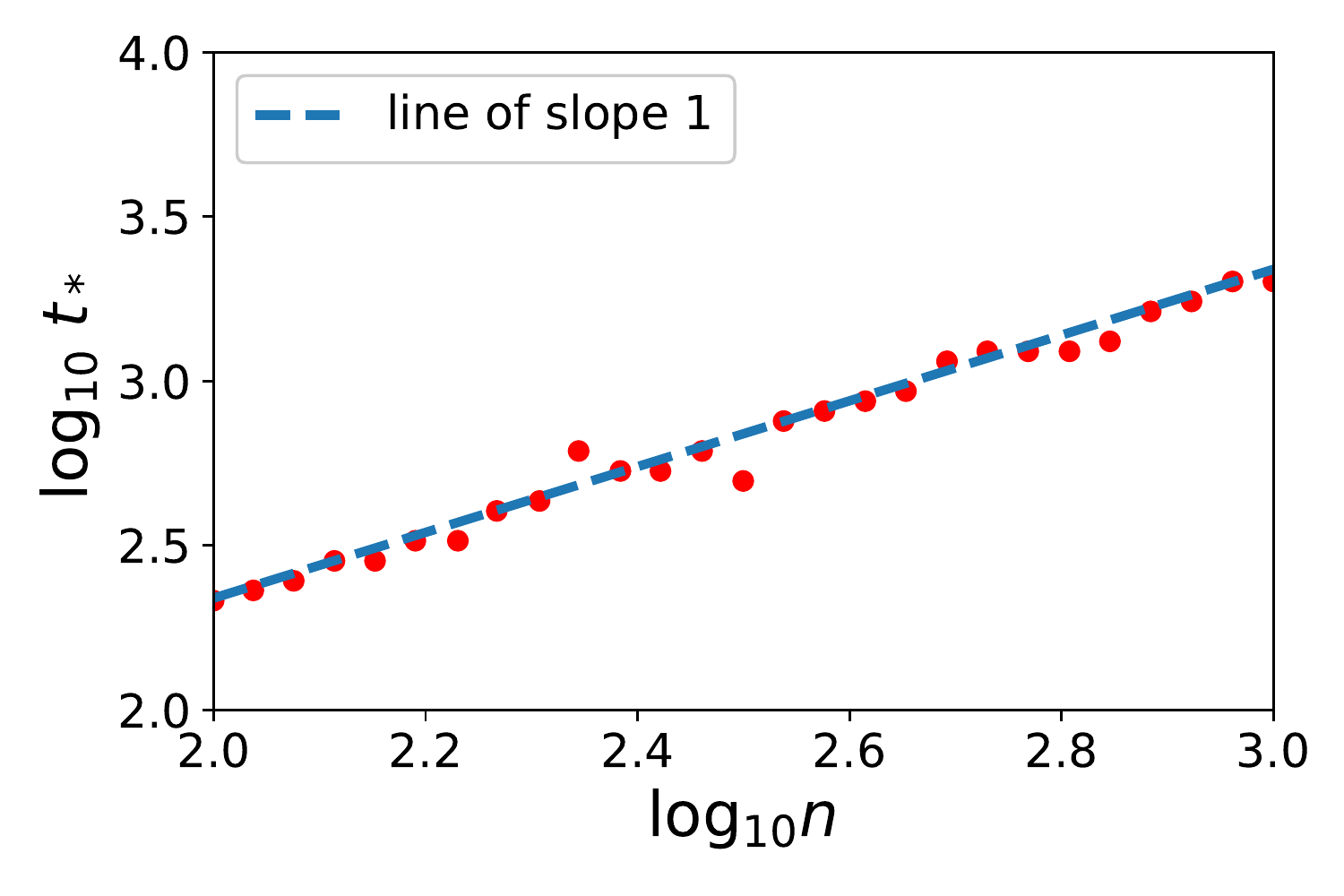} \\
\includegraphics[width=0.48\textwidth]{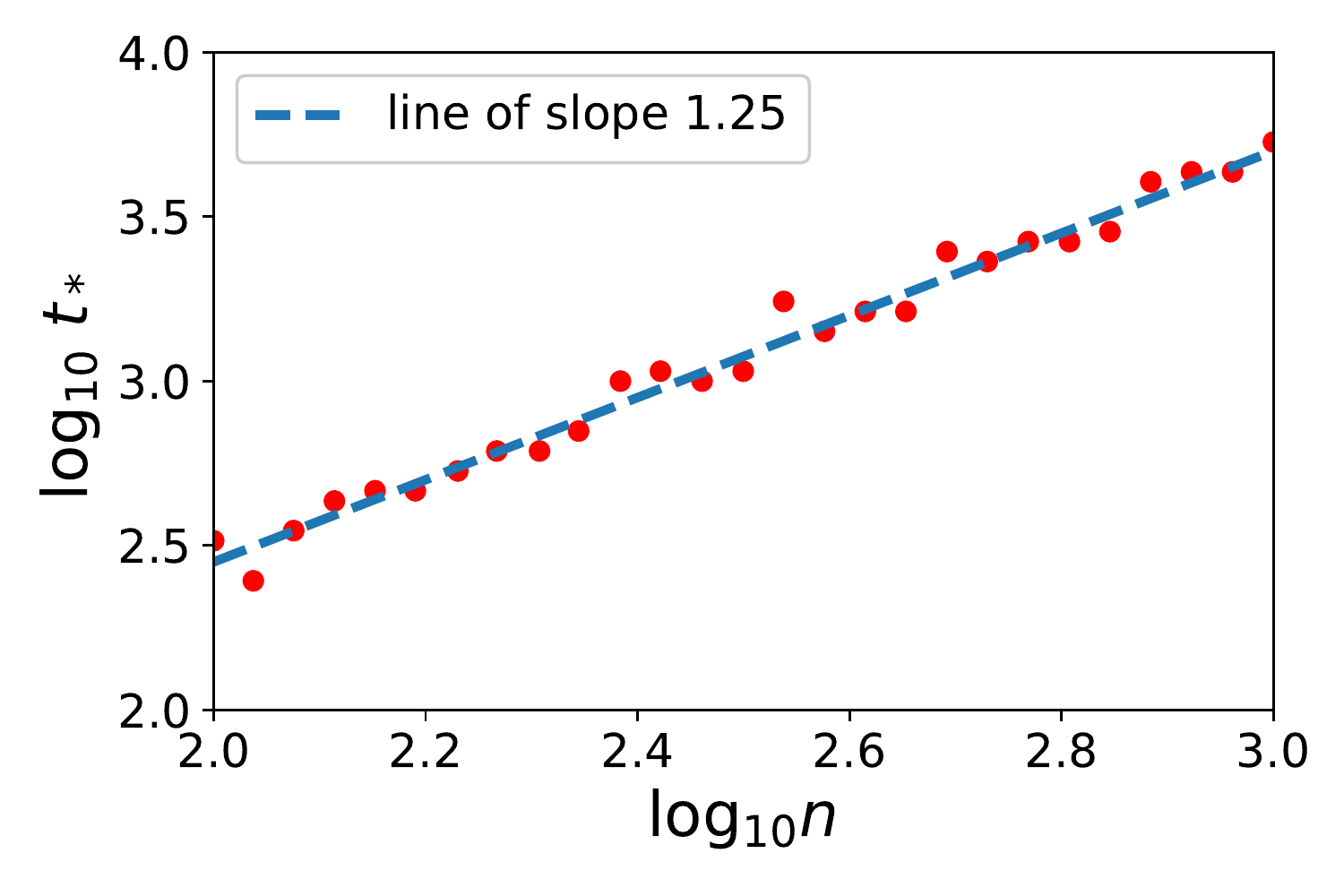}
\hspace{0.5cm}%
\includegraphics[width=0.48\textwidth]{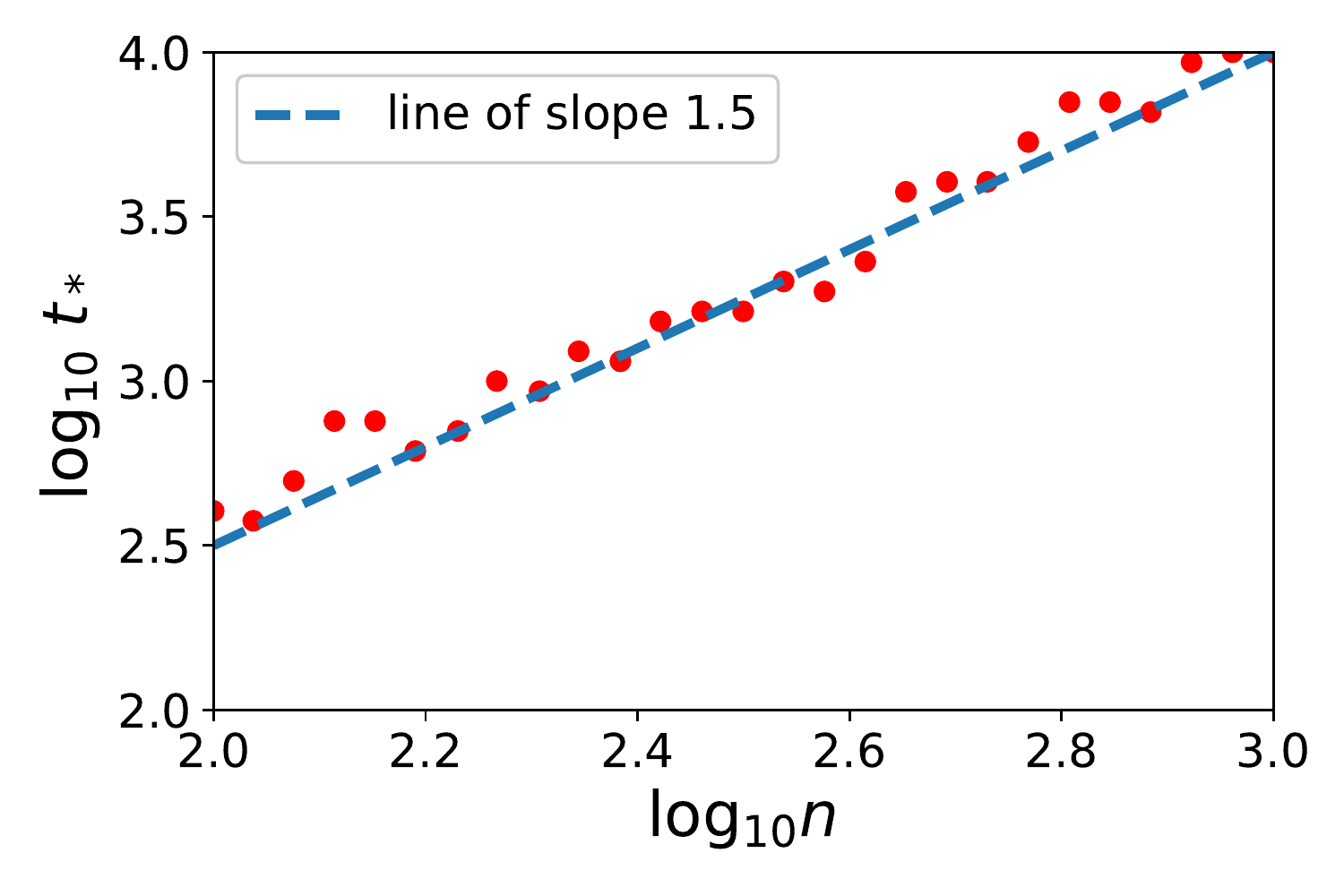}
\vspace{-0.6cm}
\caption{ \small The sampling is performed by {\bfseries cycling over the data} The four plots represent each a different configuration on the $(\alpha , r)$ plan represented in Figure \ref{fig:optimality_zones}, for $r = 1/(2\alpha) $. {\bfseries Top left} ($\alpha = 1.5$) and {\bfseries right} ($\alpha = 2$) are two easy problems (Top right is the limiting case where $r = \frac{\alpha-1}{2\alpha}$) for which one pass over the data is optimal. {\bfseries Bottom left} ($\alpha = 2.5$) and {\bfseries right} ($\alpha = 3$) are two hard problems for which an increasing number of passes is recquired. The blue dotted line are the slopes predicted by the theoretical result in Theorem \ref{thm:main_result}. }
\label{fig:t_versus_n_cycling}
\end{figure}

\begin{figure}[ht]
\footnotesize
\includegraphics[width=0.48\textwidth]{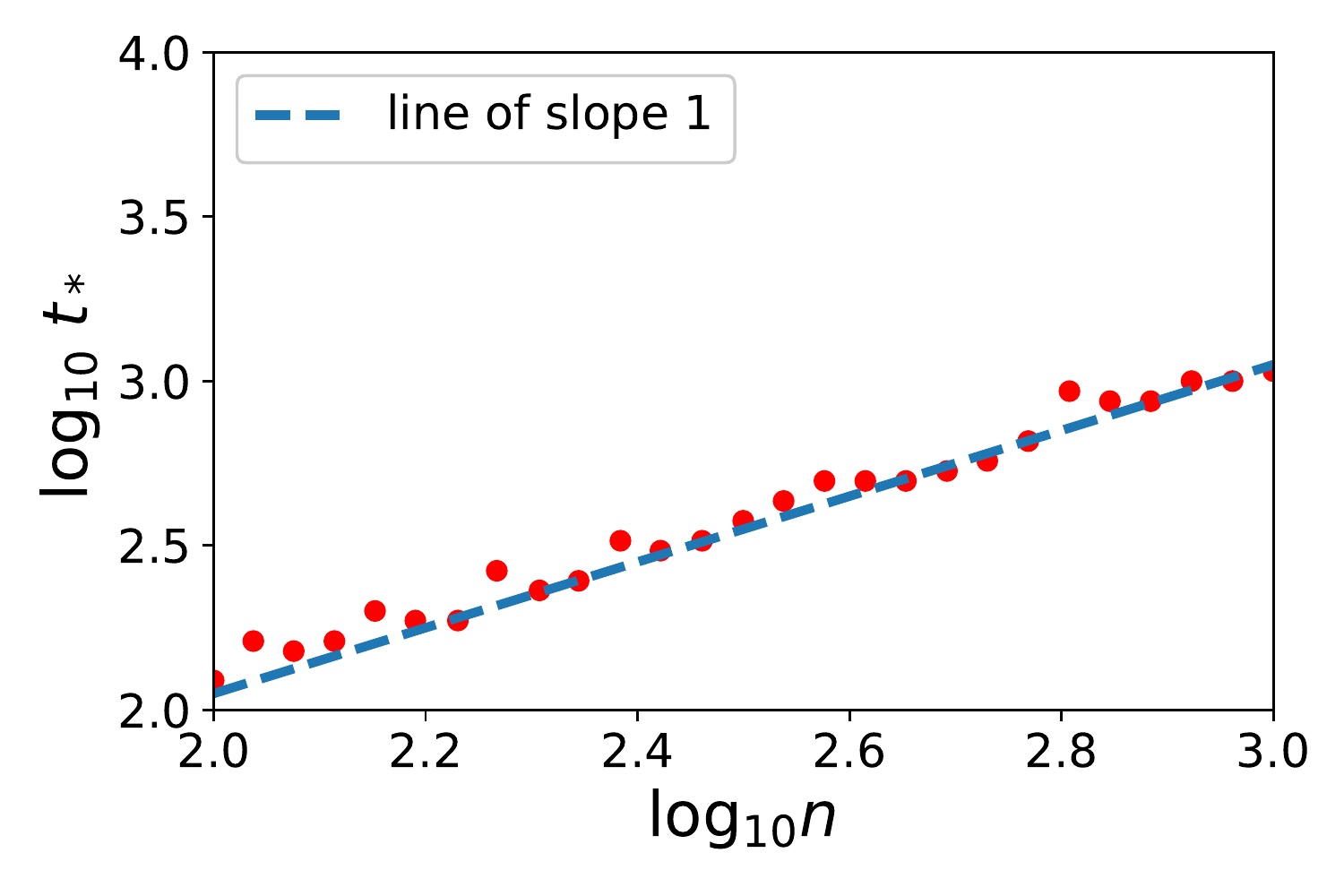}
\hspace{0.5cm}%
\includegraphics[width=0.48\textwidth]{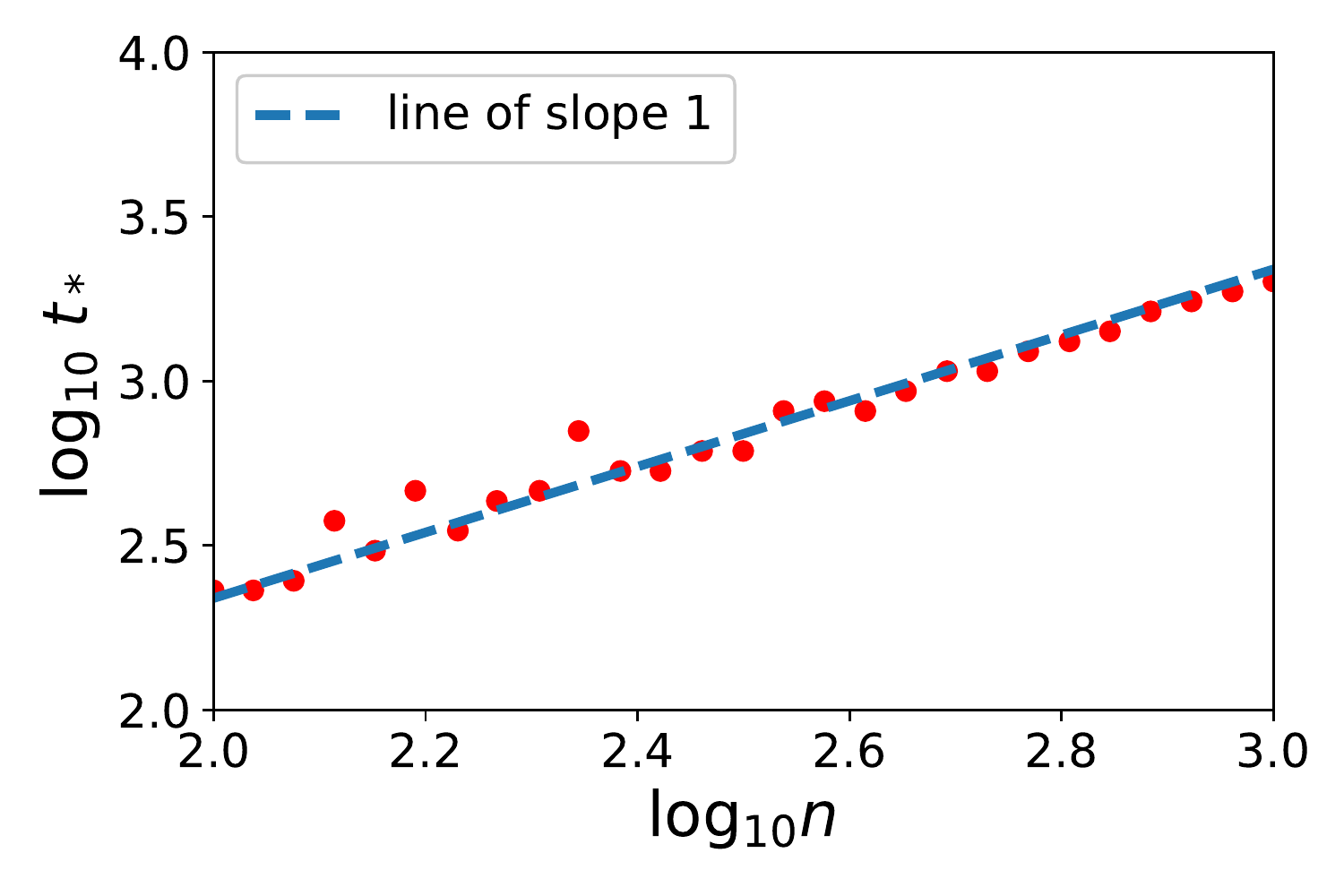} \\
\includegraphics[width=0.48\textwidth]{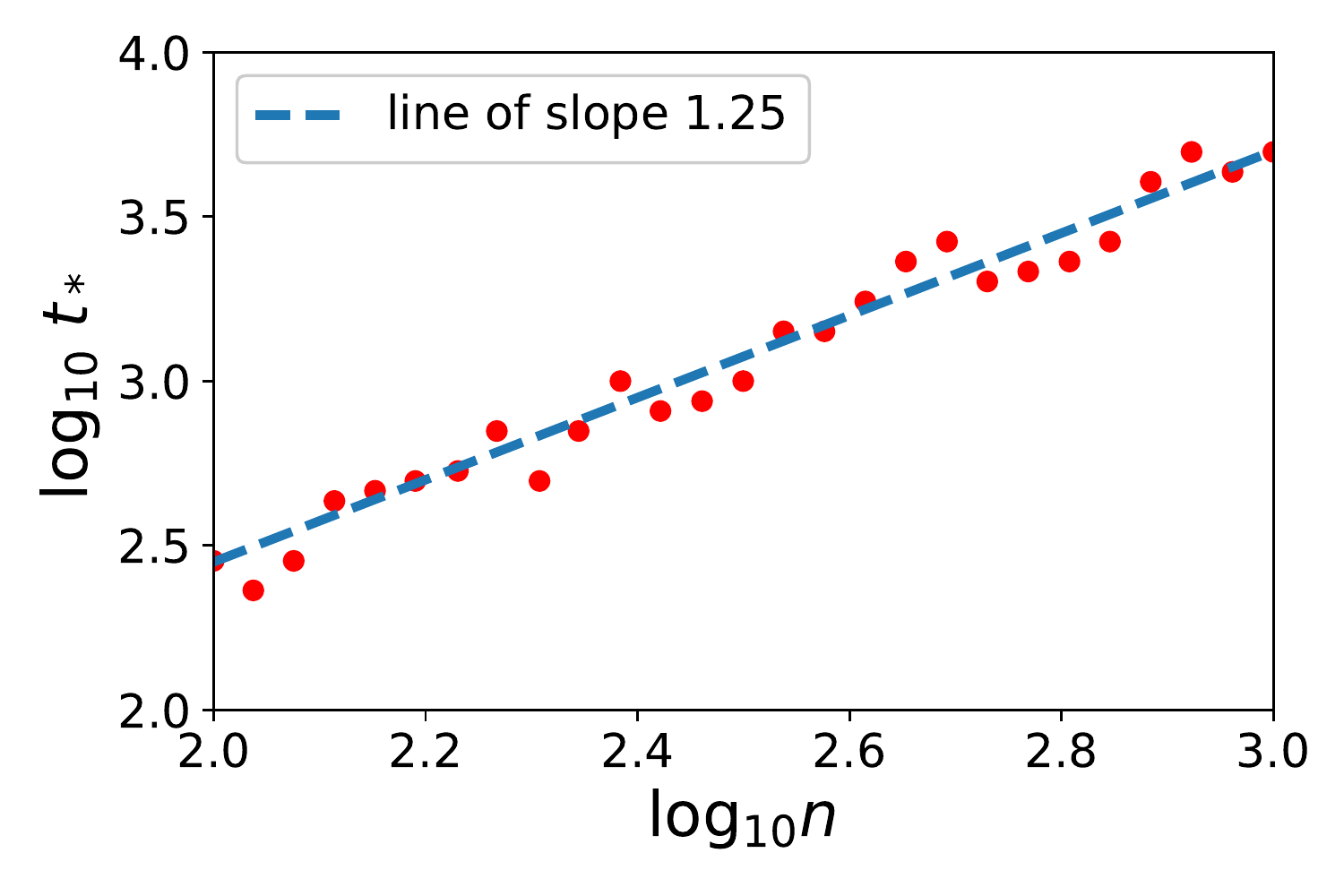}
\hspace{0.5cm}%
\includegraphics[width=0.48\textwidth]{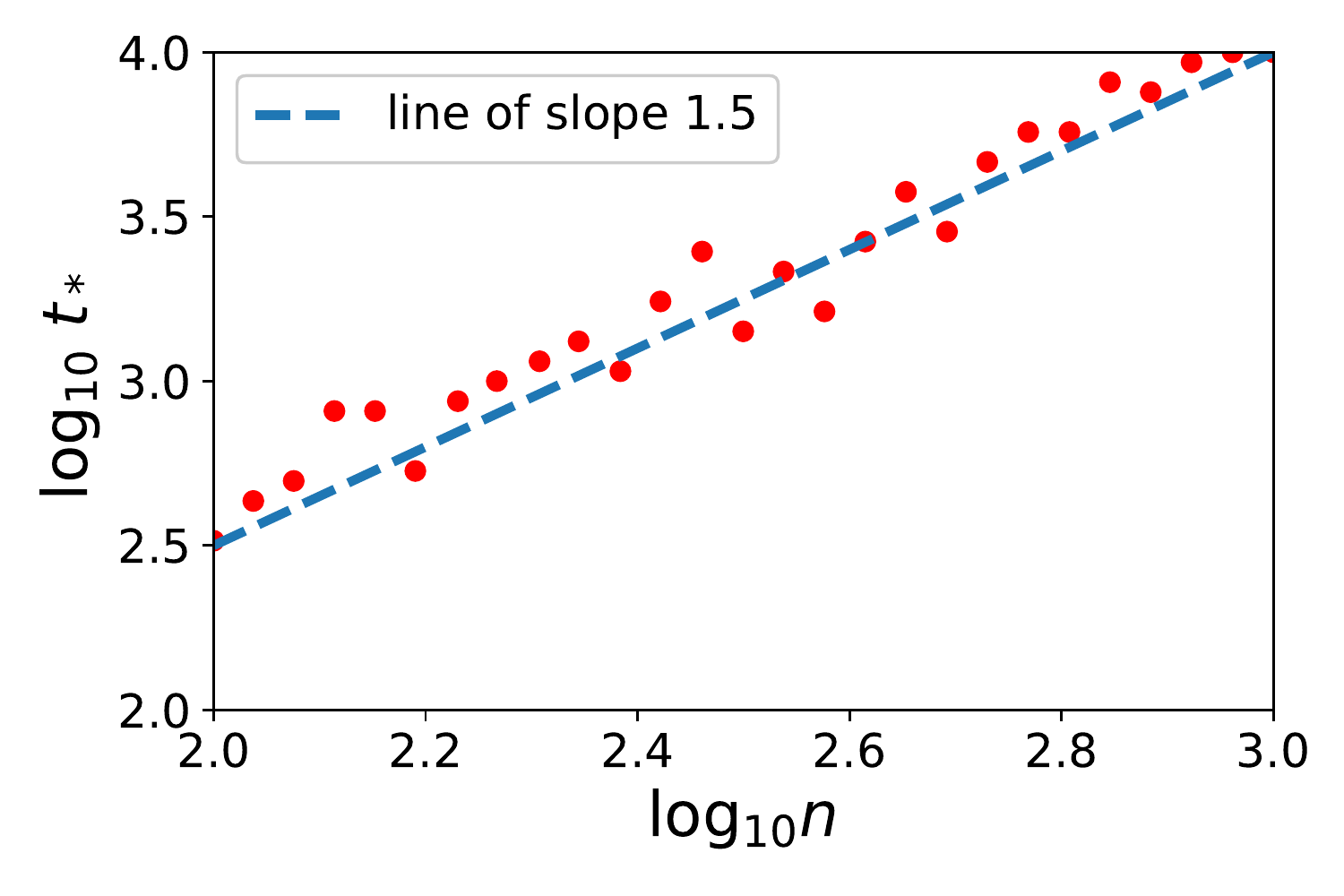}
\vspace{-0.6cm}
\caption{ \small The sampling is performed {\bfseries without replacement}. The four plots represent each a different configuration on the $(\alpha , r)$ plan represented in Figure \ref{fig:optimality_zones}, for $r = 1/(2\alpha) $. {\bfseries Top left} ($\alpha = 1.5$) and {\bfseries right} ($\alpha = 2$) are two easy problems (Top right is the limiting case where $r = \frac{\alpha-1}{2\alpha}$) for which one pass over the data is optimal. {\bfseries Bottom left} ($\alpha = 2.5$) and {\bfseries right} ($\alpha = 3$) are two hard problems for which an increasing number of passes is recquired. The blue dotted line are the slopes predicted by the theoretical result in Theorem \ref{thm:main_result}. }
\label{fig:t_versus_n_without_replacement}
\end{figure}

%
%

\end{document}